\documentclass[sigconf]{acmart}

\AtBeginDocument{%
  }

\setcopyright{acmlicensed}
\copyrightyear{2025}
\acmYear{2025}
\acmDOI{XXXXXXX.XXXXXXX}
\acmISBN{978-X-XXXX-XXXX-X/YY/MM}

\usepackage{amsmath}
\usepackage{amssymb}
\usepackage{mathtools}
\usepackage{amsfonts}
\usepackage{amsthm}
\usepackage{multirow}
\usepackage{bm}
\newcommand{\abs}[1]{\lvert#1\rvert}
\newtheorem{problem}{Problem}
\begin{document}
\newcommand{\yd}[1]{\textcolor{red}{Yushun: #1}}
\newcommand{\sli}[1]{{\leavevmode\color{olive}{\emph{[Shibo: #1]}}}}
\newcommand{\shen}[1]{\textcolor{brown}{Bolin: #1}}

\title{ATOM: A Framework of Detecting Query-Based Model Extraction Attacks for Graph Neural Networks}

\author{Zhan Cheng}
\affiliation{%
  \institution{University of Wisconsin, Madison}
  \city{Madison}
  \state{Wisconsin}
  \country{USA}
}
\email{zcheng256@wisc.edu}

\author{Bolin Shen}
\affiliation{%
  \institution{Florida State University}
  \city{Tallahassee}
  \state{Florida}
  \country{USA}}
\email{blshen@fsu.edu}

\author{Tianming Sha}
\affiliation{%
  \institution{Arizona State University}
  \city{Tempe}
  \state{Arizona}
  \country{USA}
}
\email{stianmin@asu.edu}

\author{Yuan Gao}
\affiliation{%
 \institution{University of Wisconsin, Madison}
 \city{Madison}
 \state{Wisconsin}
 \country{USA}}
\email{ygao355@wisc.edu}

\author{Shibo Li}
\affiliation{%
  \institution{Florida State University}
  \city{Tallahassee}
  \state{Florida}
  \country{USA}}
\email{sl24bp@fsu.edu}

\author{Yushun Dong}
\affiliation{%
  \institution{Florida State University}
  \city{Tallahassee}
  \state{Florida}
  \country{USA}}
\email{yushun.dong@fsu.edu}

\renewcommand{\shortauthors}{Zhan Cheng, Bolin Shen, Sha Tianming, Yuan Gao, Shibo Li, \& Yushun Dong}

\begin{abstract}
Graph Neural Networks (GNNs) have gained traction in Graph-based Machine Learning as a Service (GMLaaS) platforms, yet they remain vulnerable to graph-based model extraction attacks (MEAs), where adversaries reconstruct surrogate models by querying the victim model. Existing defense mechanisms, such as watermarking and fingerprinting, suffer from poor real-time performance, susceptibility to evasion, or reliance on post-attack verification, making them inadequate for handling the dynamic characteristics of graph-based MEA variants. To address these limitations, we propose ATOM, a novel real-time MEA detection framework tailored for GNNs. ATOM integrates sequential modeling and reinforcement learning to dynamically detect evolving attack patterns, while leveraging $k$-core embedding to capture the structural properties, enhancing detection precision. Furthermore, we provide theoretical analysis to characterize query behaviors and optimize detection strategies. Extensive experiments on multiple real-world datasets demonstrate that ATOM outperforms existing approaches in detection performance, maintaining stable across different time steps, thereby offering a more effective defense mechanism for GMLaaS environments. Our source code is available at \href{https://github.com/LabRAI/ATOM}{https://github.com/LabRAI/ATOM}.
\end{abstract}

\maketitle

\section{Introduction}
Graph Neural Networks (GNNs) \cite{scarselli2008graph,kipf2016semi} have been widely studied for modeling graph-structured data, where nodes represent entities and edges capture their relationships. Accordingly, GNNs have also demonstrated promising performance in various real-world applications, such as financial fraud detection~\cite{mao2022using,liu2021pick}, biomolecular interaction analysis \cite{bongini2022biognn,reau2023deeprank}, and personalized item recommendations \cite{gao2022graph,chang2023kgtn}. Despite its exceptional success, training GNNs has become increasingly costly due to the growing scale of both model and data. To democratize the access to powerful GNNs, Graph-based Machine Learning as a Service (GMLaaS) has emerged as a popular paradigm, which enables the model owner to provide easily accessible APIs for customers to use without disclosing the underlying GNN model.
%
%
This facilitates the broader adoption of GNNs in various domains such as e-commerce~\cite{weng2022mlaas}, healthcare~\cite{eldahshan2024optimized}, and scientific research~\cite{correia2024deepmol,zhang2021graph,shlomi2020graph}.
However, despite these advantages, GMLaaS platforms face significant security risks from model extraction attacks (MEAs) \cite{tramer2016stealing,liang2024model}. These attacks allow adversaries to query a deployed API and systematically reconstruct a surrogate model that closely mimics the target model's behavior. Recent research \cite{GUAN2024112144} has demonstrated that MEAs pose a severe threat to GMLaaS platforms, which endanger both GMLaaS providers and users, leading to financial losses and potential downstream security threats. In the financial domain, for instance, service providers can deploy GMLaaS solutions to enhance credit card fraud detection \cite{liu2021graph, xiang2023semi}. However, graph-based MEAs would enable adversaries to replicate fraud detection models, extract decision boundaries, and ultimately bypass fraud detection systems, thereby increasing the risk of large-scale financial crimes. Therefore, graph-based MEAs have emerged as a pressing security threat to GMLaaS platforms, highlighting the urgent need for robust defense mechanisms to mitigate these risks.

To counteract MEAs on GMLaaS, several mainstream defense strategies have been developed. A common defense strategy is watermarking, where model owners embed specially designed input-output patterns (as watermarks) into GNNs for ownership verification \cite{jia2021entangled, chakraborty2022dynamarks, lederer2023identifying}. Specifically, given the specially designed input, if a certain GNN model produces the same patterns in its corresponding output, it is then implied that this GNN was obtained via MEA. While effective, watermarking may degrade model accuracy \cite{lee2019defending} and still leave the model vulnerable to attackers due to its passive nature. Another related approach is fingerprinting \cite{10646643, wu2024securing}, which aims to identify stolen models by comparing their outputs to a reference model \cite{lukas2019deep, maini2021dataset}. However, both fingerprinting and watermarking are passive rather than active and can only take effect after a GNN model has been stolen. More critically, none of these methods provides proactive detection—especially in GMLaaS, where queries are sequential, adaptive, and structurally dependent. This raises a crucial question: How can we detect MEAs on GNNs proactively, rather than merely responding after an attack has already occurred? Although some DNN-based detection methods \cite{pal2021stateful, liu2022seinspect, tang2024modelguard} could be adapted for GNNs, they often fail to capture the intricate relationships between nodes in graph-structured queries. As a result, attackers can bypass detection by leveraging node relationships and evolving their query strategies. In summary, while current defenses offer some level of post-attack verification, they lack the proactive capabilities needed to detect MEAs—especially in the context of GMLaaS. This gap highlights an urgent need for novel detection approaches that can monitor the adaptive and sequential queries tailored for specific graph structures.

Despite the critical importance need of proactively detecting MEAs on GMLaaS, it is a non-trivial task and we mainly face three fundamental challenges. 
(1) \textbf{Sequential Relationship of Graph-based MEA Queries.} In the strategical querying process, an attacker typically craft each query based on the historical information of the previous output sequence. Accordingly, the evolving trajectories of queries in the input space encodes key information identifying whether the user is malicious or legitimate. However, most existing approaches rely on the hypothesis that all queries are visible for the model provider to conduct defense~\cite{juuti2019prada}, which thus makes it difficult to capture the query evolving patterns and flag potential attackers.
(2) \textbf{Dynamic Characteristics of Graph-based MEA Variants.} In GMLaaS environments, adversaries could dynamically refine their attack strategies by exploiting the structural flexibility of graph-based queries. Rather than simply replicating well-known attack signatures \cite{yu2020cloudleak, zhang2021seat}, they may adapt in real time, strategically avoiding high-risk nodes and targeting low-risk ones to evade detection. Thus, the second challenge is to design a detection framework that remains robust against evolving attack strategies. (3) \textbf{Necessity of considering multi-modal information.} Existing MEA detection methods, primarily designed for DNNs, often do not consider structural information. While effective in general cases, these methods may fail to capture the topological context of query-related nodes in GMLaaS. This is because graph-based queries involve both node attributes and topological information (e.g., multi-hop neighbors of a node). Thus, it is necessary to consider the information encoded in both modalities.

To address these challenges, we propose a novel framework, ATOM (\underline{A}ttacks de\underline{T}ector \underline{O}n G\underline{M}LaaS), for real-time detection of MEAs targeting GMLaaS environments. Specifically, to tackle the first challenge, we introduce a differential query feature encoding mechanism that analyzes changes in query features across consecutive interactions. This approach enables our framework to adapt dynamically to evolving attack behaviors by continuously monitoring and evaluating incoming queries in real-time. Next, to address the second challenge, we refine our detection strategy through a reinforcement learning approach with a normalization factor, i.e., the Proximal Policy Optimization (PPO). This allows our detection policy dynamically adjust to evolving query patterns and reveal how attackers refine their methods. We further provide a theoretical analysis of these refinements. Finally, to overcome the third challenge, we enhance each query with values reflecting a node’s structural importance, utilizing the $k$-core centrality. This enables the model to capture both local query traits and broader topological context, significantly improves its ability to distinguish between legitimate and malicious queries, even in sparsely connected scenarios. Our main contributions can be summarized as follows:
\begin{itemize}
    \item \textbf{Problem Formulation: }We provide a mathematical formulation of graph-based MEA detection in GMLaaS environments under the transductive setting, defining attack behaviors, detection objectives, and adversarial interactions.
    \item \textbf{Proposed Novel Framework: }To the best of our knowledge, ATOM is the first framework for proactive detection of graph-based MEAs in GMLaaS. Our empirical evaluations show that it outperforms existing methods adapted to this scenario.
    \item \textbf{Theoretical Analysis: }We conduct theoretical analysis on the query representation and derive formal bounds, offering a principled way to evaluate detection performance and optimize feature selection for adversarial query detection.
\end{itemize}


\section{Preliminaries}
In this section, we introduce the foundational concepts for detecting graph-based MEAs in a GMLaaS environment. The detection objective is to analyze user-submitted queries and determine whether the user is an attacker attempting to extract the deployed model. Our discussion covers the GMLaaS query-response framework, the GNN model, the objectives of both attackers and defenders and the formulation of the problem.

\subsection{Graph-based Machine Learning as a Service}

\textbf{Node-level Prediction Task.} GMLaaS systems provide a query-based interface that allows users to access pre-trained machine learning models hosted on cloud platforms~\cite{correia2024deepmol}. Node-level prediction tasks typically operate under two primary learning paradigms: the transductive setting or the inductive setting \cite{wu2024securing}. In this work, we focus on the transductive setting, where the training graph used to train the GNN model is identical to the inference graph used for serving predictions and remains unchanged throughout the service. The GMLaaS system enables users to query node-level predictions while granting access to partial graph information.

\noindent
\textbf{GMLaaS Query-Response Framework.}
In the GMLaaS setting, each user $u_{i}\in\mathcal{U}$ submits a query $q_{i,t}$ targeting a specific node $v_{i,t}\in \mathcal{V}$ within a graph $\mathcal{G}=(\mathcal{V},\mathcal{E},\boldsymbol{X})$, where $\mathcal{V}$ represents the set of nodes, $\mathcal{E}$ denotes the set of edges, and $\boldsymbol{X}\in \mathbb{R}^{|V| \times d}$ represents the node feature matrix. Upon receiving the query, the GMLaaS system provides a predicted label, denoted as $y_{i,t}=\mathcal{M}(v_{i,t})$, where $y_{i,t}\in\mathcal{C}$, and $\mathcal{C}$ is the set of possible class labels.

Additionally, user $u_{i}$ can access the one-hop subgraph $\mathcal{G}_{i,t}=(\mathcal{V}_{i,t}, \mathcal{E}_{i,t}, \boldsymbol{X}_{i,t})$ centered around the queried node $q_{i,t}$. Here, $\mathcal{V}_{i,t}=\{v_{i,t}\} \cup \{w \mid (w,v_{i,t})\in \mathcal{E}\}$ includes $q_{i,t}$ and its one-hop neighbors, $\mathcal{E}_{i,t}=\{(w,v)\in \mathcal{E} \mid w,v\in \mathcal{V}_{i,t}\}$ contains the edges connecting nodes in $\mathcal{V}_{i,t}$, and $\boldsymbol{X}_{i,t}$ represents features of nodes in $\mathcal{V}_{i,t}$.

\noindent
\textbf{User and Query Sequences.} Consider a set of users denoted as $\mathcal{U}=\{u_{1}, u_{2}, \cdots, u_{M}\}$, where each user submits queries independently. The query history of a user $u_{i}\in\mathcal{U}$ is represented as a query sequence $\mathcal{Q}_{i}=\{q_{i,1}, q_{i,2}, \cdots, q_{i,T_{i}}\}$, where $T_{i}$ denotes the total number of queries made by $u_{i}$. Since queries arrive sequentially, $T_{i}$ also represents the total time steps of queries for user $u_{i}$.

\subsection{Graph Neural Networks (GNNs)}
Acting as the backbone of our proposed framework, a Graph Neural Network (GNN) $\mathcal{M}$ is trained on the static graph $\mathcal{G}$ for a specific downstream learning task. The basic operation of GNN between $l$-th and $(l+1)$-th layer can be formulated as follows:\begin{align}
\boldsymbol{h}_{v}^{(l+1)}=\sigma(\text{COMBINE}(\boldsymbol{h}_{v}^{(l)},f(\{\boldsymbol{h}_{u}^{(l)}:u\in\mathcal{N}(v)\}))),
\end{align}where $\boldsymbol{h}_{v}^{(l+1)}$ and $\boldsymbol{h}_{v}^{(l)}$ represent the embedding of node $v$ at $l$-th and $(l+1)$-th layer correspondingly. The node feature matrix $\boldsymbol{X}$ serves as the input to the GNN, where each node feature $\boldsymbol{x}_v$ initializes the corresponding hidden representation $\boldsymbol{h}_v^{(0)}$. Given the adjacency matrix $\boldsymbol{A}$, the neighbor set of node $v$ is denoted as $\mathcal{N}(v)$. The aggregation function $f(\cdot)$ gathers information from the neighbors of $v$, and the combining function $\text{COMBINE}(\cdot)$ integrates this information with the current hidden representation $\boldsymbol{h}_v^{(l)}$. An activation function $\sigma(\cdot)$ (e.g., ReLU), is applied to introduce non-linearity. Given the output of the last GNN layer by matrix $\boldsymbol{Z}\in\mathbb{R}^{n\times c}$, the prediction $\hat{\boldsymbol{Y}}$ of GNN can be written as $\text{softmax}(\boldsymbol{Z})\in\mathbb{R}^{n\times c}$ for node classification, and $\text{sigmoid}(\boldsymbol{Z}^T\boldsymbol{Z})\in\mathbb{R}^{n\times n}$ for link prediction \cite{kipf2016variational}.

\subsection{Adversary's Objective}
The adversary's goal is to reconstruct a surrogate model $\mathcal{M}'$ that closely approximates the behavior of the victim GNN model $\mathcal{M}$. This is achieved by systematically querying the GMLaaS system and collecting query-response pairs.

\noindent
\textbf{Adversary's Knowledge.} 
Following the attack taxonomy in \cite{wu2022model}, we assume that the attacker possesses partial knowledge of the graph’s structure and attributes. For instance, in a social network system, an attacker may access partial user connections and attributes through public profiles. Specifically, the attacker $u_{i}$ can access a subgraph $\mathcal{G}'\subset \mathcal{G}$. At time step $t$, $u_{i}$ submits query $q_{i,t}$ to access the one-hop subgraph $\mathcal{G}_{i,t}\subset \mathcal{G}'$ and the node features $\boldsymbol{X}_{i,t}$, where $\mathcal{G}_{i,t}=(\mathcal{V}_{i,t}, \mathcal{E}_{i,t}, \boldsymbol{X}_{i,t})$. The attacker can dynamically refine their query strategy based on information obtained from prior queries.

\noindent
\textbf{Extracted Model Training.}
The attacker trains the extracted model $\mathcal{M}'$ by minimizing the prediction error between the victim model $\mathcal{M}$ and $\mathcal{M}'$. This objective is formulated as:\begin{align}
\min_{\mathcal{M}'}\mathbb{E}_{v\in V}[\mathcal{L}(\mathcal{M}(v),\mathcal{M}'(v))],
\end{align}where $\mathcal{L}$ is the loss function measuring the prediction difference.

\subsection{Defender's Objective}
The defender’s goal is to detect adversarial users by classifying users based on their query sequences. This requires designing a detection function $Z$, which assigns a classification label: $d_{i,T}=Z(\mathcal{Q}_{i,T})$,
where $d_{i,T} \in \{0,1\}$, with $0$ representing a legitimate user and $1$ representing an attacker. Formally, the defender aims to learn an optimal function $Z^*$ that maximizes detection accuracy while minimizing false positives and false negatives:\begin{align}
Z^* = \arg\max_{Z} \mathbb{E} [\mathbb{I}(Z(\mathcal{Q}_{i,T}) = y_i)],
\end{align}where $y_i$ is the truth label of user $u_i$, $\mathbb{I}(\cdot)$ is the indicator function.

\subsection{Problem Statement}
\begin{problem} 
    \textbf{Graph-based MEA detection in GMLaaS environments under the transductive setting.} Let $\mathcal{G}=(\mathcal{V},\mathcal{E},\boldsymbol{X})$ be a static attributed graph. A GNN $\mathcal{M}$ is deployed on an GMLaaS platform under the transductive setting. Each user $u_i \in \mathcal{U}$ submits a query sequence $\mathcal{Q}_{i}=\{q_{i,t}\}_{t=1}^{T_i}$. Our goal is to design a detection function $Z$ that assigns a label $d_{i,T}$ to user $u_i$ based on their query sequence $\mathcal{Q}_{i,T}$ up to time step $T$, aiming to maximize the expected classification accuracy: $Z^* = \arg\max_{Z} \mathbb{E} [\mathbb{I}(Z(\mathcal{Q}_{i,T}) = y_i)]$, so that attackers and legitimate users can be accurately classified. 

\end{problem}

\begin{figure*}[ht]
\vskip 0.2in
    \begin{center}
\centerline{\includegraphics[width=\textwidth]{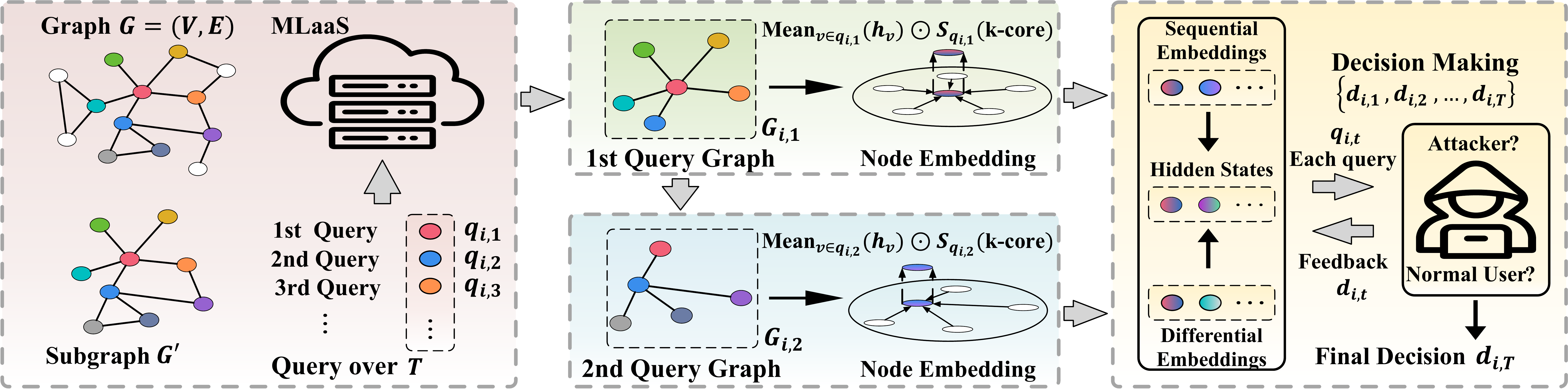}}
\caption{An illustration of the framework with the query behavior and the detection mechanism.}
\label{figure1_framwork}
\end{center}
\vskip -0.2in
\end{figure*}

\section{Methodology}
\subsection{Framework Overview}
An overview of the proposed framework is shown in Figure~\ref{figure1_framwork}. Specifically, it consists of two modules: (1) \textbf{Attack Simulation.} This module generates realistic model extraction attack sequences to serve as training data for the detection model. To achieve this, we integrate active learning techniques to mimic adversarial query behaviors under realistic GMLaaS constraints. (2) \textbf{Attack Detection.} This module consists of query embedding, a sequential network, and a reinforcement learning-based detection mechanism. It processes query sequences and classifies users as attackers or legitimate users based on their query behaviors.

\subsection{Attack Simulation}
A major challenge in constructing a reliable detection mechanism is obtaining high-fidelity training data that accurately represent real-world attacker behaviors. Instead of relying on passive observation, we proactively simulate realistic MEAs through the following steps.

\subsubsection{Active learning based Attacks}Since existing Graph-based MEAs are relatively limited, we adapt active learning (AL) \cite{settles2009active} to construct realistic attack query sequences, since both AL and MEAs share a common objective: maximizing knowledge extraction from a model while operating under strict query constraints. We simulate attacks using three representative algorithms:

\noindent
\textbf{AGE \cite{cai2017active}} (Active Exploration-Based Query Strategy). At each time step $T$, AGE selects a node $v_T$ based on a scoring function $S(v_T)$, which integrates: Information entropy (uncertainty), Information density (node importance), and Graph centrality (network influence). To align with GMLaaS constraints, we modify AGE with the average highest score within one-hop subgraph of $v_T$, defined as:\begin{align}
S_{avg}(v_{T})=\frac{\sum_{v\in V_{T}}S(v)+S(v_{T})}{V_{T}+1},
\end{align}This ensures that query sequences reflect real-world constraints on node accessibility. The generated query sequence follows a descending order based on $S_{\text{avg}}$.

\noindent
\textbf{GRAIN \cite{zhang2021grain}} (Influence Maximization-Based Query Strategy). At each time step $T$, GRAIN selects a node $v_T$ to maximize the score function $S(\mathcal{G}'_s)$, where:\begin{align}S(\mathcal{G}'_{s})=\frac{\abs{\sigma(\mathcal{G}'_{s})}}{\abs{\hat{\sigma}}}+\gamma\frac{D(\mathcal{G}'_{s})}{\hat{D}}.\end{align}Here, $\sigma(\mathcal{G}'_s)$ represents the influence spread of the selected subgraph $\mathcal{G}'_s$, and $D(\mathcal{G}'_s)$ measures query diversity. The generated query sequence follows a descending order based on $S(\mathcal{G}'_s)$.

\noindent
\textbf{IGP \cite{zhang2022information}} (Label-Informed Query Strategy). At each time step $T$, IGP selects the next node $v_T$ by assuming the pseudo-label with the highest confidence in its softmax output $\hat{y}_T$, thereby maximizing the entropy change in its neighborhood. To improve efficiency, we first pre-filter nodes using a ranking score:\begin{align}
s=\alpha\cdot \mathcal{P}_{centrality}+(1-\alpha)\cdot\mathcal{P}_{entropy},
\end{align}Only the top-ranked nodes are selected for querying, reducing query overhead while maximizing model information extraction.

\subsubsection{Query Sequence Generation}
We utilize the above attack simulation strategies to train surrogate models $\mathcal{M}'$ with corresponding query sequences $\mathcal{Q}_{\mathcal{M}'_j}$. However, not all extracted sequences are considered valid attacks. We apply a quality threshold $F_{\text{threshold}}$, retaining only high-fidelity attack sequences:\begin{align}\mathcal{Q}_{attack}=\{\mathcal{Q}_{\mathcal{M}'_{1}},\mathcal{Q}_{\mathcal{M}'_{2}},\cdots, \mathcal{Q}_{\mathcal{M}'_{H}}\},
\end{align}where $F(\mathcal{M}'_j)>F_{\text{threshold}}$. For training balance, we also include legitimate user query sequences:\begin{align}
\mathcal{Q}_{normal}=\{\mathcal{Q}_{1},\mathcal{Q}_{2},\cdots ,\mathcal{Q}_{N}\}
\end{align}All sequences are labeled (attack = $1$, normal = $0$), shuffled, and assigned to a set of users $\mathcal{U}$, where $|\mathcal{U}|=|\mathcal{Q}_{attack}|+|\mathcal{Q}_{normal}|$. Thus, for each user $u_i \in \mathcal{U}$, a query sequence $\mathcal{Q}_i = \{q_{i,1}, q_{i,2}, \dots, q_{i,T_i}\}$ is generated for training the attack detection model.

\subsection{Attack Detection}
Attack detection in GMLaaS environments is more than a binary classification problem. Real attackers adapt over time steps, modifying their queries based on model responses to evade detection. A detection system that classifies queries individually, without considering their sequential nature or strategic dependencies, is insufficient. Furthermore, the detection mechanism must be resilient, continuously refining its strategy as attack patterns evolve.

\subsubsection{Sequences Embedding}
At time step $T$, each query $q_{i,T}$ is transformed into an embedding $h_{i,T}$, incorporating both node features and graph topological information:\begin{align}
h_{i,T}=\frac{\sum_{v\in \mathcal{V}_{i,T}}h_{v}}{|\mathcal{V}_{i,T}|}\odot \mathcal{S} (\frac{\log(p_{v_{i,T}})}{\log(p_{max})}),
\end{align}where $h_{v}$ represents node embeddings obtained from the GMLaaS model $\mathcal{M}$, $p_{v_{i},T}$ is the $k$-core value of the central node $v_{i,T}$, and $p_{max}$ is the maximum $k$-core value in graph $\mathcal{G}$. Here, $\mathcal{S}(x)$ is a scaling function defined as:\begin{align}
\mathcal{S}(x)=1+\lambda\cdot(\sigma(\lambda\cdot x)-0.5)\times 2,
\end{align}where $\lambda$ is a hyperparameter controlling the effect of topological scaling, ensuring that $h_{i,T}$ is modulated based on graph structure while keeping variations within $[1-\lambda,1+\lambda]$. This embedding mechanism ensures that detection captures structural dependencies, making it harder for adversaries to exploit low-connectivity nodes for stealthy model extraction.

\subsubsection{Sequential Modeling}
Model extraction attacks evolve over time steps—each query is part of a larger, strategic attack sequence. To capture temporal dependencies, we enhance a classic Gated Recurrent Unit (GRU) \cite{cho2014learning} with: (1) Differential input encoding, which highlights query-to-query variations (2) A fusion gate, selectively incorporating past and present query features. (3) A mapping matrix, adjusting hidden states based on past classification decisions. At time step $T$, we compute the differential input $\delta_{i,T}$ as $\delta_{i,T}=h_{i,T}-h_{i,T-1}$, where $h_{i,0}=0$. We introduce a fusion gate $g_T$, which determines how much of the current query embedding $h_{i,T}$ and its differential input $\delta_{i,T}$ should be retained:\begin{align}
g_{T}=\sigma(\boldsymbol{W}_{g}\cdot \text{Concat}(\delta_{i,T},h_{i,T})+b_{g}),
\end{align}where $\boldsymbol{W}_g$ and $b_g$ are learnable parameters. The input is given by:\begin{align}
x_{T}=g_{T}\odot \delta_{i,T}+(1-g_{T})\odot h_{i,T}.
\end{align}This ensures that detection is based on the "story" behind a sequence of queries, rather than treating them as isolated requests.

The sequential hidden state is updated with the GRU mechanism:\begin{align}
h_{i,T}^{seq} = [(1 - z_T) \odot h_{i,T-1}^{seq} + z_T \odot \tilde{h}_T]^{T}\cdot \boldsymbol{m}_{i,T-1},
\end{align}where $z_T$ is the update gate, $\tilde{h}_T$ is the candidate state,  and $\boldsymbol{m}_{i,T-1}$ is a mapping matrix introduced to adjust the hidden state based on historical classification actions. Here, the mapping matrix $\boldsymbol{m}_{i,T-1}$ is computed as:\begin{align}
\boldsymbol{m}_{i,T-1}=\boldsymbol{W}_{a}\cdot p_{d_{i,T-1}}+b_{a},
\end{align}where $p_{d_{i,T-1}}$ represents the classification probabilities from the PPO-based reinforcement learning module, and $W_a, b_a$ are learnable transformation matrices. This adjustment ensures that past classification decisions influence future query analysis, making detection more adaptive to evolving attack strategies.

\subsubsection{Decision Making}
Static detection rules cannot adapt to emerging attack strategies. To enable continuous learning, we integrate reinforcement learning (RL) via Proximal Policy Optimization (PPO) \cite{schulman2017proximal}. At time step $T$, the system observes a state $s_{i,T}$, selects an action $d_{i,T} \in \{0,1\}$ (attacker or legitimate user), and receives a reward $R(s_{i,T}, d_{i,T})$ based on classification correctness:\begin{align}
R(s_{i,T}, d_{i,T})=\begin{cases}
    R_w(s_{i,T}, d_{i,T}), \\
    R_{\text{penalty}}\quad\text{for classification bias},
\end{cases}
\end{align}where $R_w(s_{i,T}, d_{i,T})$ is defined as:\begin{align}
R_w(s_{i,T}, d_{i,T}) = 
\begin{cases} 
w_{\text{TP}}, & \text{if } d_{i,T} =1 \text{ and }l=1, \\
w_{\text{TN}}, & \text{if } d_{i,T} =0 \text{ and } l=0, \\
-w_{\text{FN}}, & \text{if } d_{i,T} =0 \text{ and } l=1, \\
-w_{\text{FP}}, & \text{if } d_{i,T} =1 \text{ and } l=0,
\end{cases}
\end{align}The bias penalty $R_{\text{penalty}}$ is applied when the model overwhelmingly classifies users as attackers or normal users, defined as:\begin{align}
R_{\text{penalty}}=-p,\quad\text{where }p>w_{\text{FN}}>\max\{w_{\text{TP}},w_{\text{TN}},w_{\text{FP}}\}
\end{align}

\section{Theoretical Analysis}
In this section, we establish the theoretical foundation of our proposed framework by linking the graph-based query interaction scenario to fundamental mathematical concepts. Specifically, we model user behavior as a dominating set problem on a weighted graph, where the objective is to balance coverage and weight minimization. In this setting, legitimate users seek to maximize coverage efficiently, whereas attackers attempt to maximize the total weight of accessed nodes while minimizing coverage to evade detection. To address this challenge, we demonstrate that incorporating first-order and second-order differences in query embeddings is crucial for capturing adversarial behaviors, particularly in dynamic query sequences. Additionally, we provide a probabilistic interpretation of ATOM in the appendix.

\subsection{Query as a Dominating Set Problem}
In this subsection, we interpret the process of accessing a subgraph by a user as constructing a dominating set. The objective for a normal user is to cover necessary nodes while minimizing resource costs, typically quantified by node weights. However, adversarial users often follow a different strategy: they attempt to maximize the total weight of accessed nodes while keeping the coverage rate low to remain undetected. Theorem $\ref{1}$ formalizes this trade-off.
\begin{theorem}
\label{1}
  Consider a covering graph $\mathcal{D}$ in the graph $\mathcal{G}=(\mathcal{V},\mathcal{E})$, aiming to cover at least $\beta\in[0,1]$ percent nodes of $\mathcal{G}$, while minimizing $\sum_{u\in \mathcal{A}}w(u)$, where $w(u)$ is the weight of node $u$ and $\mathcal{A}$ represents the set of nodes not being covered, then the maximum covering percentage is given by \begin{align}\beta\le \min\{\frac{\abs{\mathcal{D}}-\frac{W}{\Bar{w_{\mathcal{A}}}}}{\frac{n}{\delta}-\frac{W}{\Bar{w_{\mathcal{A}}}}}, \frac{\abs{\mathcal{D}}\cdot \delta}{n}\}.\end{align} Here, $\abs{\mathcal{D}}$ represents the number of nodes in $\mathcal{D}$, $W$ is the total weight in $\mathcal{G}$, $\Bar{w_{\mathcal{A}}}$ represents the average weight in $\mathcal{A}$ and $\delta$ is the smallest degree for nodes in $\mathcal{D}$.
\end{theorem}Our results indicate that increasing the minimum degree of the covered subgraph while querying would enhance the coverage, which could be adopted to implement a high-quality MEA. Based on this observation, we integrate $k$-core values into the query embeddings to prioritize structurally significant nodes. This ensures that the detected subgraphs remain well-connected, thereby constraining the attacker's ability to manipulate coverage.

\subsection{Incremental Changes in Query Behavior}
In this subsection, we aim to model how queries change over time steps. Specifically, we examine incremental changes in graph coverage and node weights. We define the first-order difference to measure how the weight of uncovered nodes evolves as new nodes are queried, capturing gradual shifts in user behavior. Proposition $\ref{2}$ relates the weight reduction per node to the change in coverage.
\begin{proposition}
\label{2}
     Consider a changing graph $\mathcal{D}_{t-1}$ and $\mathcal{D}_{t}$ in the graph $\mathcal{G}=(\mathcal{V},\mathcal{E})$, where $\mathcal{D}_{t-1}\subset \mathcal{D}_{t}\subset \mathcal{G}$, achieving at least $\beta_{t-1}$ covering rate with the lowest degree $\delta_{t-1}$ and at most $\beta_{t}$ covering rate with the lowest degree $\delta_{t}$, respectively. Also, suppose that $A_{t-1}$ and $A_{t}$ represent the set of nodes that are not covered, with the corresponding weight $W_{{\mathcal{A}}_{t-1}}$, $W_{{\mathcal{A}}_{t}}$ and the average weight $\Bar{w_{A_{t}}}$, $\Bar{w_{A_{t-1}}}$. We then get \begin{align}
    \frac{\Delta W_{\mathcal{A}}}{\Delta\abs{\mathcal{D}}}\le\frac{(\beta_{t}-\beta_{t-1}) \cdot W}{\abs{\mathcal{D}_{t-1}}(1-\frac{\delta_{t-1}}{\delta_{t}})},
    \end{align}and $\delta_{t}>\delta_{t-1}$. Here, $\abs{\mathcal{D}_{t-1}},\abs{\mathcal{D}_{t}}$ represents the number of nodes in $\mathcal{D}_{t-1},\mathcal{D}_{t}$, $W$ is the total weight in $\mathcal{G}$.
\end{proposition}Here, $\frac{\Delta W_{\mathcal{A}}}{\Delta\abs{\mathcal{D}}}$ acts as a first-order difference, quantifying how the weight of uncovered nodes evolves as new nodes are queried. This provides a direct measure of the trade-off between weight minimization and coverage expansion, enabling the model to capture gradual shifts in adversarial behavior.

However, first-order differences alone may fail when attackers target high-weight nodes while minimizing coverage expansion. In such cases, the weight reduction per added node may fluctuate, making it necessary to consider second-order differences to capture variations in how these changes occur over time steps.

\subsection{Strategy Shifts in Query Behavior}
To capture fluctuations in the rate of coverage expansion and weight reduction discussed above, we introduce the second-order differences to measure the change in first-order differences. Specifically, we aim to address the challenge when attackers adjust their query strategy by alternating between targeting high-weight nodes and optimizing coverage. The second-order difference can reveal these shifts, while the first-order difference may appear stable in this scenario. Proposition \ref{3} demonstrates that it serves as a key indicator of irregular access patterns. The second-order difference in weight is given as follows.
\begin{proposition}
\label{3}
     Consider changing graphs $\mathcal{D}_{t-1},\mathcal{D}_{t}$ and $\mathcal{D}_{t+1}$ in the graph $\mathcal{G}=(\mathcal{V},\mathcal{E})$, where $\mathcal{D}_{t-1}\subset \mathcal{D}_{t}\subset \mathcal{D}_{t+1}\subset \mathcal{G}$, achieving at least $\beta_{t-1}$ covering rate with the lowest degree $\delta_{t-1}$ and at most $\beta_{t}$ covering rate with the lowest degree $\delta_{t}$. Also, there is at least $\beta_{t}'$ covering rate and at most $\beta_{t+1}$ covering rate at time step $t$ and $t+1$. Suppose that $\mathcal{A}_{t-1},\mathcal{A}_{t}$ and $\mathcal{A}_{t+1}$ represent the set of nodes that are not covered, with the corresponding weight $W_{\mathcal{A}_{t-1}},W_{\mathcal{A}_{t}}$, $W_{\mathcal{A}_{t+1}}$ and the average weight $\Bar{w_{\mathcal{A}_{t-1}}},\Bar{w_{\mathcal{A}_{t}}}$, $\Bar{w_{\mathcal{A}_{t+1}}}$. We then get \begin{align}
    \frac{\Delta^2 W_{\mathcal{A}}}{\Delta^2\abs{\mathcal{D}}}\ge \left|\frac{W_{d}}{n}\frac{\Delta \delta_{t}}{\Delta \beta_{t}}-\frac{W\delta_{t+1}}{\abs{\mathcal{D}_{t}}}\frac{\Delta \beta_{t+1}}{\Delta \delta_{t+1}}\right|,
    \end{align}and $\delta_{t+1}>\delta_{t}>\delta_{t-1}$. Here, $\abs{\mathcal{D}_{t-1}},\abs{\mathcal{D}_{t}}$ and $\abs{\mathcal{D}_{t+1}}$ represent the number of nodes in $\mathcal{D}_{t-1},\mathcal{D}_{t}$ and $\mathcal{D}_{t+1}$, $W$ is the total weight in $\mathcal{G}$, assuming $W_{\mathcal{A}_{t-1}}-W_{\mathcal{A}_{t}}\ge W_{d}$.
\end{proposition}To integrate this into our framework, we define the second-order difference in query embeddings as $\delta_{i,T}=h_{i,T}-h_{i,T-1}$, which captures temporal variations. These help identify adversarial patterns where attackers subtly shift their behavior to avoid detection.

\subsection{Importance of Second-Order Differences}
To present the importance of introducing second-order differences, we establish a condition in Theorem \ref{4} when the second-order differences contribute more significantly to detection than the first-order ones. Specifically, our analysis derives a threshold. If the traditional detection mechanism could exceed this threshold, we say the second-order differences are well worth being considered.
\begin{theorem}
\label{4}
    Consider changing graphs $\mathcal{D}_{t-1},\mathcal{D}_{t}$ and $\mathcal{D}_{t+1}$ with graph $\mathcal{G}=(\mathcal{V},\mathcal{E})$, where $\mathcal{D}_{t-1}\subset \mathcal{D}_{t}\subset \mathcal{D}_{t+1}\subset \mathcal{G}$, aiming at achieving at least $\beta_{t-1}$ covering rate with the lowest degree $\delta_{t-1}$ and at most $\beta_{t}$ covering rate with the lowest degree $\delta_{t}$. Also, there is at least $\beta_{t}'$ covering rate and at most $\beta_{t+1}$ covering rate at time step $t$ and $t+1$. Suppose that $\mathcal{A}_{t-1},\mathcal{A}_{t}$ and $\mathcal{A}_{t+1}$ represent the set of nodes that are not covered, with the corresponding weight $W_{\mathcal{A}_{t-1}},W_{\mathcal{A}_{t}}$, $W_{\mathcal{A}_{t+1}}$ and the average weight $\Bar{w_{\mathcal{A}_{t-1}}},\Bar{w_{\mathcal{A}_{t}}}$, $\Bar{w_{\mathcal{A}_{t+1}}}$. The second-order differences become essential , that is, $\frac{\Delta^2 W_{\mathcal{A}}}{\Delta^2\abs{\mathcal{D}}}\ge \frac{\Delta W_{\mathcal{A}}}{\Delta\abs{\mathcal{D}}}$ holds when \begin{align}
    W_{d}\ge \frac{n\delta_{t+1}\Delta \beta_{t}}{\abs{\mathcal{D}_{t-1}}\Delta \delta_{t}}(\frac{\Delta \beta_{t}}{\Delta \delta_{t}}+\frac{\Delta \beta_{t+1}}{\Delta \delta_{t+1}})W,
    \end{align}and $\delta_{t+1}>\delta_{t}>\delta_{t-1}$. Here, $\abs{\mathcal{D}_{t-1}},\abs{\mathcal{D}_{t}}$ and $\abs{\mathcal{D}_{t+1}}$ represent the number of nodes in $\mathcal{D}_{t-1},\mathcal{D}_{t}$ and $\mathcal{D}_{t+1}$, $W$ is the total weight in $\mathcal{G}$, assuming $W_{\mathcal{A}_{t-1}}-W_{\mathcal{A}_{t}}\ge W_{d}$.
\end{theorem}This insight highlights the importance of dynamically adjusting detection strategies based on the observed query behavior. To leverage this, we incorporate a fusion gate $g_{T}$ that adaptively balances the first and second-order differences. This ensures that the model remains robust against evolving attack strategies, adjusting its detection focus as needed. We prove this theorem by empirical evaluations in section~\ref{5.3} and section~\ref{5.4}.

\section{Experimental Evaluations}
We conduct a series of experiments to evaluate the performance of the proposed framework. Specifically, we seek to address the following research questions: \textbf{RQ1: }How effectively can the proposed model capture attacks compared to baseline methods? \textbf{RQ2: }How do the individual components contribute to the overall performance of the proposed model? \textbf{RQ3: }How does hyperparameter $\lambda$ influence the performance of the proposed model?

\subsection{Experiment Setup}
\textbf{Downstream Task and Datasets. }We adopt the node classification task and evaluate the model on five widely used benchmark datasets: Cora, Citeseer, PubMed, Cornell, and Wisconsin. These datasets can be categorized into two distinct types based on their structural characteristics. In the first three datasets, nodes represent research publications, and edges denote citation relationships. In the remaining datasets, nodes correspond to webpages, and edges indicate hyperlinks between them. Unlike citation networks, webpage networks often exhibit different topological properties, making them valuable for testing the generalization ability of our approach. To simulate real-world adversarial scenarios, we implement MEAs on all five datasets during our experiments.

\noindent
\textbf{GMLaaS Models. }We train a two-layer GCN as the target model within a GMLaaS setting. The model configuration is as follows: The hidden layer is configured with $16$ features with ReLU activation, while the output layer uses softmax activation for classification. We optimize the model using the Adam optimizer with a learning rate of $0.01$, a weight decay of $0.0005$, and $200$ training epochs. Following the transductive setting, the graph used during training is identical to the one used for inference.

\noindent
\textbf{Adversarial Knowledge. }As previously discussed, we assume the attacker has partial knowledge of both the node attributes and graph structure. The adversary is allowed to access a single node and its one-hop subgraph at a time. Under this constraint, we implement three attack algorithms based on AL learning: AGE, GRAIN, and IGP, ensuring a realistic evaluation of our model's robustness against graph-based attacks.

\noindent
\textbf{Baselines. }To thoroughly assess the effectiveness of the proposed detection framework, we compare it against a diverse set of baseline models, categorized as follows. We first employ commonly used neural network architectures for sequential and classification tasks: \textit{Simple MLP \cite{rumelhart1986learning}: }A fundamental feedforward neural network for classification. \textit{RNN \cite{hopfield1982neural}: }A recurrent model that processes sequential data but struggles with long-term dependencies. \textit{LSTM \cite{hochreiter1997long}: }An improved recurrent architecture incorporating gating mechanisms for long-range information retention. \textit{Transformer \cite{vaswani2017attention}: }A self-attention-based architecture that effectively captures long-range dependencies in sequential data. Since MEAs detection can be framed as a time-series classification task \cite{ismail2019deep}, we incorporate several models from this domain: \textit{Crossformer \cite{wang2021crossformer}: }Employs a cross-scale attention mechanism to capture temporal dependencies at different scales. \textit{Autoformer \cite{wu2021autoformer}: }Integrates self-correlation and autoregressive structures to model periodic trends in time series. \textit{TimesNet \cite{wu2022timesnet}: }Reformulates time series as a multi-period representation, capturing temporal variations both within and across periods. \textit{PatchTST \cite{nie2022time}: }Treats time-series segments as receptive fields in a convolutional framework, extracting multi-scale temporal patterns. \textit{Informer \cite{zhou2021informer}: }Uses a sparse attention mechanism to efficiently model long-range dependencies in sequential data. \textit{iTransformer \cite{liu2023itransformer}: }A Transformer-based model that balances global temporal dependencies and local feature extraction through sequence decomposition. We also compare our framework with existing DNN-based detection approaches designed specifically for MEA detection in GMLaaS environments: \textit{PRADA \cite{juuti2019prada}: }Detects adversarial behavior by analyzing statistical deviations in sequential API queries. \textit{VarDetect \cite{pal2021stateful}: }Uses a Variational Autoencoder (VAE) to model user query distributions and identify anomalies.

\begin{table*}[h!]
\centering
\vspace{-1mm}
\caption{Performance Comparison between ATOM and baselines on different metrics and datasets. The best results are in bold.}
\vspace{-2mm}
\renewcommand{\arraystretch}{1.16}
\label{table1-ref}
\resizebox{\textwidth}{!}{
\begin{tabular}{lcccccccccc}
\hline
\textbf{Metrics} & \multicolumn{5}{c}{\textbf{F1 score}} & \multicolumn{5}{c}{\textbf{Recall}} \\
\cmidrule(lr){2-6} \cmidrule(lr){7-11}
\textbf{Dataset} & \textbf{Wisconsin} & \textbf{Cornell} & \textbf{Cora} & \textbf{Citeseer} & \textbf{PubMed} & \textbf{Wisconsin} & \textbf{Cornell} & \textbf{Cora} & \textbf{Citeseer} & \textbf{PubMed} \\
\hline
MLP & $24.32 \pm 15.72$ & $56.71 \pm 9.60$ & $33.56 \pm 13.74$ & $39.06 \pm 9.76$ & $56.44 \pm 7.93$ & $23.57 \pm 8.57$ & $63.12 \pm 14.32$ & $34.15 \pm 14.37$ & $39.17 \pm 10.34$ & $58.42 \pm 11.47$ \\
RNN-A & $66.24 \pm 8.83$ & $52.08 \pm 15.43$ & $ 58.25 \pm 4.70$ & $59.05 \pm 12.90$ & $60.99 \pm 10.92$ & $53.57 \pm 9.78$ & $43.75 \pm 15.93$ & $52.36 \pm 6.31$ & $52.50 \pm 17.50$ & $51.92 \pm 13.73$ \\
LSTM-A & $53.12 \pm 10.36$ & $49.48 \pm 10.03$ & $57.58 \pm 3.28$ & $57.45 \pm 10.89$ & $54.24 \pm 16.93$ & $56.57 \pm 9.70$ & $44.53 \pm 15.53$ & $50.94 \pm 3.53$ & $49.17 \pm 11.46$ & $44.23 \pm 18.34$ \\
Transformer-A& $72.59 \pm 7.05$ & $60.24 \pm 7.97$ & $55.22 \pm 6.22$ & $70.10 \pm 9.77$ & $61.62 \pm 10.28$ & $60.71 \pm 9.04$ & $54.51 \pm 9.09$ & $49.37 \pm 8.22$ & $65.00 \pm 14.34$ & $51.58 \pm 12.19$ \\
\hline
Crossformer & $\small{75.76\pm0.24}$ & $\small{67.79\pm1.42}$ & $\small{75.12\pm4.35}$ & $\small{59.69\pm4.60}$ & $\small{45.21\pm17.12}$ & $\small{79.29\pm2.14}$ & $\small{61.43 \pm 16.2}$ & $\small{63.87\pm14.9}$ & $\small{70.71\pm15.2}$ & $\small{46.09\pm13.41}$ \\
Autoformer & $\small{56.75\pm7.67}$ & $\small{79.07\pm4.80}$ & $\small{65.15\pm17.5}$ & $\small{60.76\pm9.72}$ & $\small{76.44\pm10.8}$ & $\small{55.71\pm16.1}$ & $\small{90.71\pm9.61}$ & $\small{57.42\pm20.1}$ & $\small{62.86\pm19.6}$ & $\small{75.22\pm17.7}$ \\
iTransformer & $\small{56.89\pm6.90}$ & $\small{55.98\pm9.45}$ & $\small{60.30\pm3.12}$ & $\small{62.08\pm9.76}$ & $\small{63.86\pm3.97}$ & $\small{60.00\pm16.0}$ & $\small{59.29\pm17.3}$ & $\small{51.77\pm16.7}$ & $\small{66.43\pm14.8}$ & $\small{63.04\pm13.9}$ \\
TimesNet & $\small{66.63\pm5.90}$ & $\small{81.79\pm4.82}$ & $\small{82.24\pm6.79}$ & $\small{59.22\pm4.36}$ & $\small{61.36\pm0.18}$ & $\small{68.57\pm11.0}$ & $\small{92.43\pm6.59}$ & $\small{84.52\pm12.7}$ & $\small{59.29\pm13.4}$ & $\small{53.91\pm12.3}$ \\
PatchTST& $\small{61.04\pm8.40}$ & $\small{62.96\pm 5.11}$ & $\small{65.79\pm12.9}$ & $\small{57.56\pm8.11}$ & $\small{79.51\pm5.39}$ & $\small{62.86\pm19.1}$ & $\small{64.29\pm16.5}$ & $\small{59.84\pm10.2}$ & $\small{54.29\pm12.6}$ & $\small{82.61\pm13.3}$ \\
Informer & $\small{53.47\pm0.42}$ & $\small{81.36\pm4.91}$ & $\small{72.01\pm1.63}$ & $\small{49.24\pm1.74}$ & $\small{65.17\pm4.37}$ & $\small{54.29\pm14.3}$ & $\small{91.29\pm4.14}$ & $\small{71.45\pm14.8}$ & $\small{52.86\pm17.1}$ & $\small{63.04\pm10.3}$ \\
\hline
PRADA& $\small{19.01\pm1.73}$ & $\small{12.34\pm0.89}$ & $\small{11.23\pm1.05}$ & $\small{13.57\pm1.52}$ & $\small{16.78\pm1.24}$ & $\small{17.54\pm1.42}$ & $\small{13.45\pm0.76}$ & $\small{14.56\pm0.98}$ & $\small{15.89\pm1.13}$ & $\small{18.95\pm1.37}$\\
VarDetect & $\small{64.28\pm1.32}$ & $\small{68.23\pm24.7}$ & $\small{61.95\pm0.10}$ & $\small{53.15\pm1.24}$ & $\small{55.16\pm2.19}$ & $\small{43.29\pm1.14}$ & $\small{60.74\pm14.8}$ & $\small{41.58\pm0.71}$ & $\small{52.17\pm0.93}$ & $\small{49.47\pm2.75}$\\
\hline
ATOM & $\bm{81.48\pm1.02}$ & $\bm{89.66\pm0.97}$ & $\bm{86.88\pm0.93}$ & $\bm{78.89\pm1.39}$ & $\bm{83.24\pm0.68}$ & $\bm{90.91\pm3.67}$ & $\bm{96.15\pm2.11}$ & $\bm{93.65\pm1.07}$ & $\bm{85.71\pm1.71}$ & $\bm{92.42\pm2.02}$ \\
\hline
\end{tabular}
}
\end{table*}

\noindent
\textbf{Ablated Models. }To assess the contribution of each individual component in the proposed framework, we conduct ablation studies by selectively removing or modifying specific modules. We define three ablated model variants: \textit{(1) Replacing the enhanced GRU with a standard GRU. }This ablation removes the fusion gate, allowing us to evaluate the importance of differential input mechanisms. \textit{(2) Replacing the proposed $k$-core-based embeddings with simple mean embeddings. }This experiment highlights the effectiveness of our scaling function, which is further explored in the Evaluation of Parameter Test section. \textit{(3) Removing the mapping matrix. }This variation investigates the role of the mapping matrix in improving robustness by incorporating historical decision-making.

\noindent
\textbf{Evaluation Metrics. }We evaluate the proposed framework using the following performance metrics: \textit{(1) Detection Effectiveness. }We evaluate both the F1 score and recall metrics to measure attack detection performance. A higher F1 score and recall indicate better classification accuracy while minimizing false negatives. \textit{(2) Ablation Study. }We compare the F1 scores of the full model and its ablated versions. This experiment also serves as an empirical validation of Theorem \ref{4}, particularly in analyzing the performance difference between the enhanced GRU and the standard GRU. \textit{(3) Parameter Sensitivity. }We evaluate the impact of different values of the scaling factor $\lambda$ on the F1 score. This analysis provides insights into how parameter tuning influences detection performance.

\subsection{Evaluation of Detection Effectiveness}
\label{5.2}
To address \textbf{RQ1}, we evaluate the effectiveness of ATOM by comparing its performance against multiple baselines. Since no existing methods are specifically designed for graph-based MEA detection in GMLaaS environments, we construct a diverse set of baselines to ensure a fair and comprehensive comparison. Specifically, we adopt classical classification models, replace the fusion GRU in ATOM with alternative sequential networks, and introduce time-series classification models to account for the temporal structure of the task. To maintain clarity, we append the suffix "-A" to the names of sequential baselines in Table~\ref{table1-ref}. Additionally, we incorporate DNN-based detection strategies to examine the limitations of general MEA detection methods. To further evaluate real-time detection performance, we simulate progressive query arrival by assessing all models with $\{25\%, 50\%, 75\%, 100\%\}$ of the query sequences. Here, we highlight the strongest-performing models within each category of the baselines. The performance of Transformer-A, Informer, VarDetect, and ATOM on Cora is visualized in Figure~\ref{figure2_step}, while results for other baselines and datasets are provided in the Appendix. We summarize our observations below: (1) ATOM consistently achieves competitive performance across all baselines. It prioritizes recall value while maintaining a strong F1 score, which is particularly important for MEA detection. Specifically, ATOM achieves a well-balanced distinction between attackers and legitimate users, enhancing its practical applicability. (2) DNN-based MEA detection methods do not generalize well to graph-based MEAs. In particular, PRADA exhibits the weakest performance among all models, as it assumes user queries follow a normal distribution, which is not realistic in real-world attack scenarios. Similarly, VarDetect, despite successfully encoding queries into a latent space, performs comparably to time-series classification models, underscoring the difficulty of directly extending existing DNN-based detection techniques to MEA detection in GNNs. (3)  In real-time settings, ATOM outperforms all baselines at every percentage of queries processed while maintaining low variance. This highlights the advantage of processing queries sequentially rather than treating them as independent samples. During the first $25\%$ of queries processed, all models exhibit similar performance due to the limited available information. However, as more queries are processed, the performance of ATOM rapidly improves, showcasing its ability to adapt dynamically to evolving attack strategies.

\begin{figure}
    \centering
    \vspace{-4mm}
    \begin{minipage}[htp]{\linewidth}
        \centering
        \includegraphics[width=\textwidth, keepaspectratio, trim={1cm 3cm 1cm 2cm}, clip]{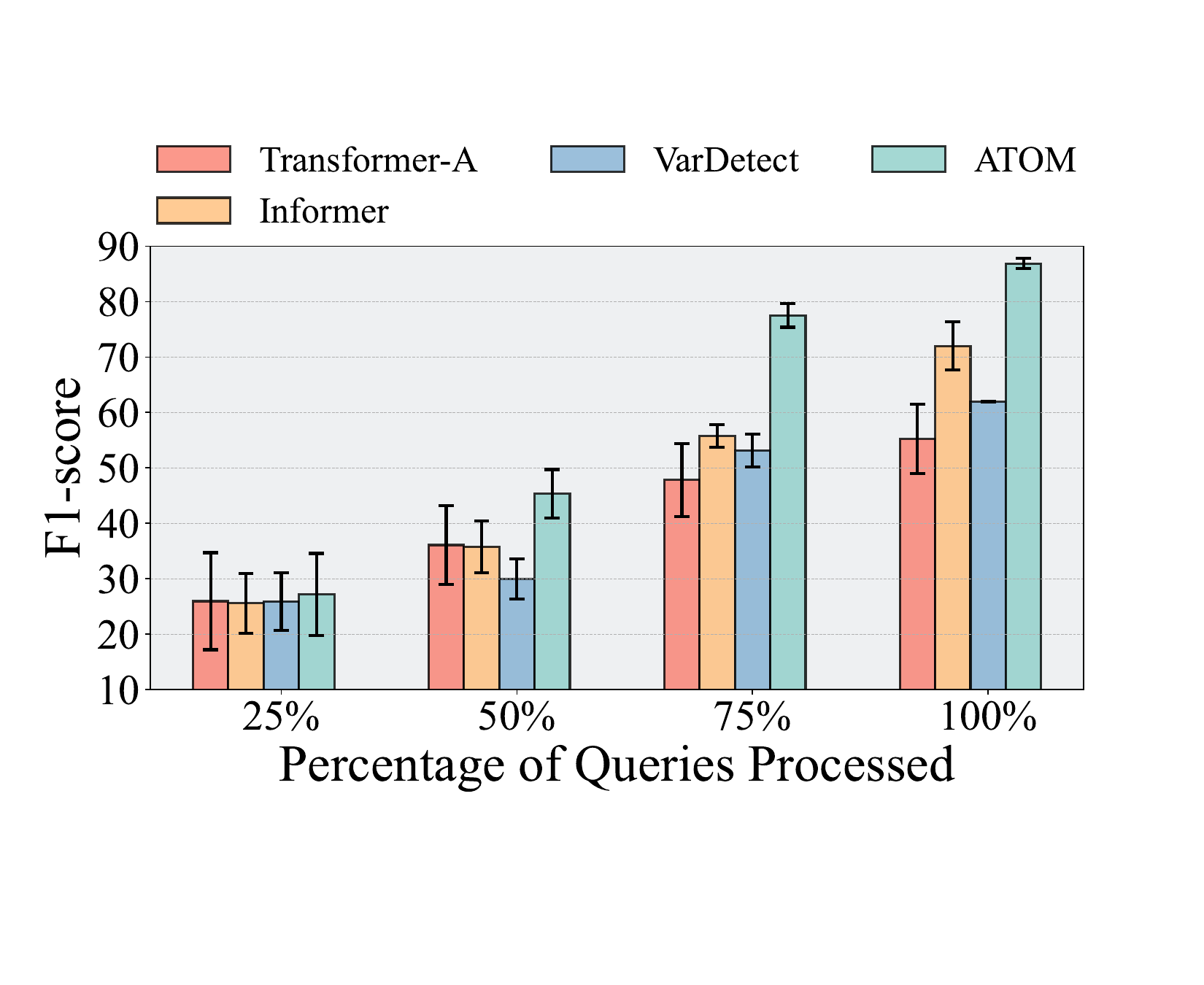}
    \end{minipage}
        \vspace{-4mm}
    \caption{Performance of Representative Models Over Sequential Query Processing on Cora.}
    \label{figure2_step}
        \vspace{-5mm}
\end{figure}

\subsection{Evaluation of Ablation Study}
\label{5.3}
To answer \textbf{RQ2}, we conduct a comprehensive evaluation of the full ATOM model and its ablated counterparts to assess the contribution of individual components. Specifically, we present their F1 scores across five benchmark datasets in Table~\ref{table_ablation}. Through this analysis, we derive the following key observations: (1) The incorporation of second-order differences enhances MEA detection. This observation empirically supports Theorem $\ref{4}$, demonstrating that capturing higher-order temporal variations in user query sequences provides valuable discriminative features for identifying adversarial behavior. By modeling the second-order differences, the framework effectively captures subtle yet critical variations in query patterns, which would otherwise be overlooked by first-order representations. (2) The $k$-core embedding significantly improves performance compared to standard embedding. As shown in Table~\ref{table_ablation}, the $k$-core embedding leads to a noticeable enhancement in classification, reinforcing its effectiveness in MEA detection. This improvement stems from its ability to extract topological features from the graph, which are particularly useful in distinguishing between queries. Furthermore, in the Evaluation of Parameter Test section, we provide a detailed discussion of how different levels of $k$-core embeddings affect model performance, offering additional insights into the optimal choice of $\lambda$. (3) The mapping matrix acts as a specialized normalization mechanism for the hidden state, facilitating network convergence. More specifically, the mapping matrix functions as a scaling transformation applied to the hidden state, with its scaling factor dynamically controlled by the previous time step. This mechanism enhances the model’s robustness by mitigating unstable fluctuations in hidden representations, thereby promoting stable and efficient convergence during training. The empirical results further confirm that the inclusion of the mapping matrix contributes to improved generalization performance.

\begin{table}[t]
\small
\centering
\renewcommand{\arraystretch}{1.1}
\caption{F1 scores from the ablated model on Cora, PubMed, and CiteSeer. The best results are highlighted in bold.}
\resizebox{0.46\textwidth}{!}{
\begin{tabular}{lccc}
    \hline
    \textbf{Model}               & \textbf{Cora}       & \textbf{Citeseer}       & \textbf{PubMed}   \\
    \hline
    ATOM                         & \bm{$86.88$}       & \bm{$78.89$}       & \bm{$83.24$}         \\
    Standard GRU                 & $80.64$           & $71.97$           & $75.47$                 \\
    Simple Embeddings            & $67.92$           & $60.87$           & $54.83$              \\
    No Mapping Matrix            & $81.54$           & $74.71$           & $79.94$                 \\
    \hline
\end{tabular}
}
\label{table_ablation}
\vskip -3ex
\end{table}

\subsection{Evaluation of Parameter Test}
\label{5.4}
To answer \textbf{RQ3}, we investigate the impact of varying the scaling factor $\lambda$ on the proposed model across different datasets. Specifically, we systematically adjust $\lambda$ in ATOM and report its F1 score in Figure~\ref{lambda}. We record our results by $\lambda \in \{0, 0.001, 0.01, 0.1, 0.5, 1, 2, 5, 10\}$. The following key observations are drawn from this evaluation: (1) Incorporating $\lambda$ as a scaling factor for the $k$-core value significantly enhances the ability of sequential networks to process queries. This result highlights the crucial role of topological information in refining node embeddings. By scaling the $k$-core value appropriately, ATOM effectively integrates structural properties into its feature representations, thereby improving its capacity to detect graph-based MEAs. The empirical results suggest that leveraging well-calibrated structural embeddings strengthens the temporal modeling capability of sequential networks, leading to more robust attack detection. 
\begin{figure}
    \centering
        \vspace{-3.5mm}
    \begin{minipage}[htp]{\linewidth}
        \centering
        \includegraphics[width=\textwidth, keepaspectratio, trim={1cm 4cm 1cm 3cm}, clip]{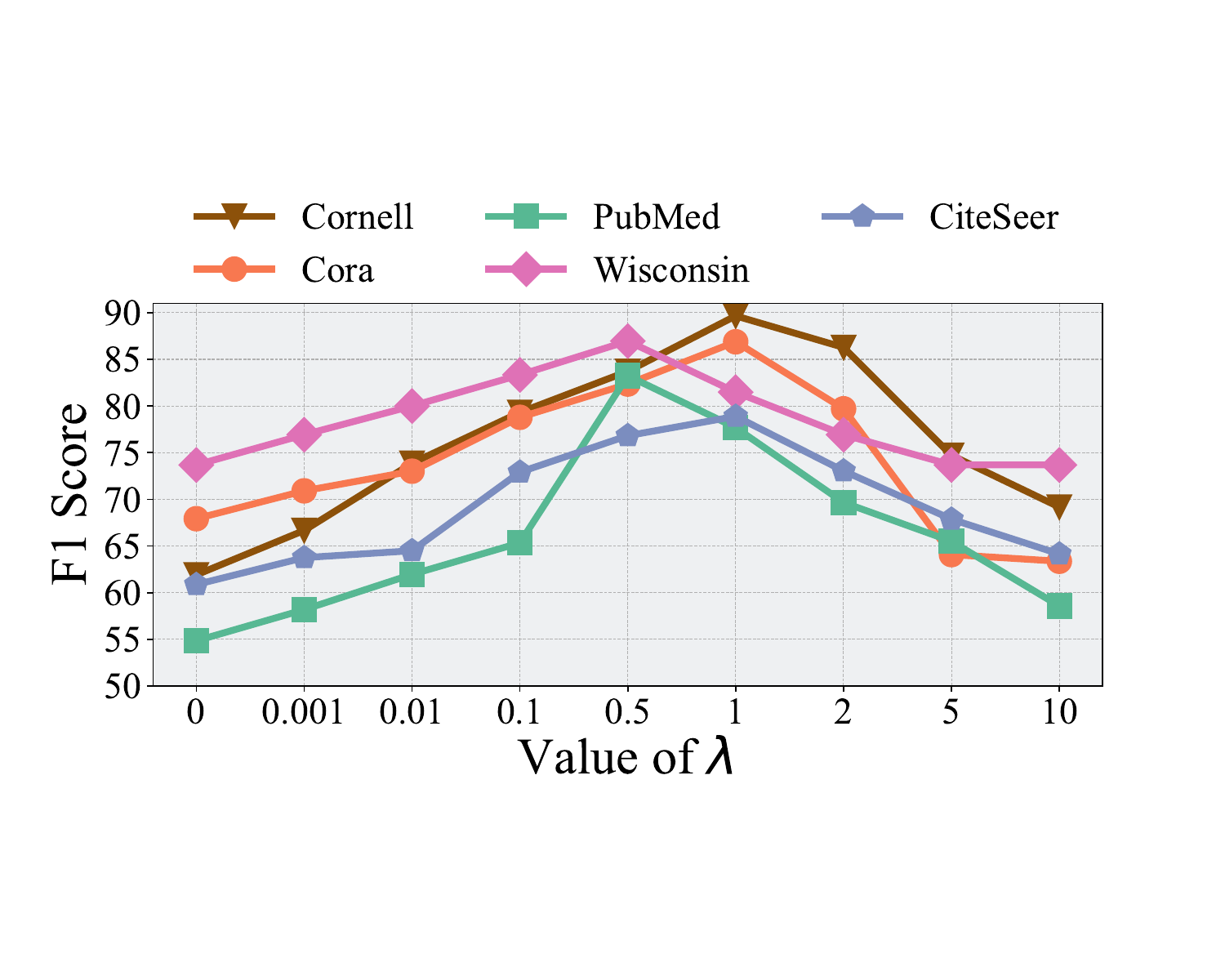}
    \end{minipage}
        \vspace{-3mm}
    \caption{Impact of the adjustment factor $\lambda$ in ATOM.}
    \label{lambda}
        \vspace{-6mm}
\end{figure}
(2) The selection of $\lambda$ is critical; both excessively small and overly large values negatively impact model performance. When $\lambda$ is too small, the influence of graph structural information on node embeddings is minimal, making ATOM's performance comparable to that of models that solely rely on raw node attributes. This limitation prevents ATOM from fully exploiting the underlying network topology, thereby restricting its ability to capture adversarial patterns. When $\lambda$ is set to a moderate value, ATOM achieves its best performance, with an improvement of up to $51.8\%$ for PubMed in the F1 score compared to cases where no scaling factor is introduced ($\lambda = 0$). In particular, when $\lambda \approx 0.5$ to $1$, ATOM effectively balances local attribute information with global topological properties, leading to more discriminative representations for MEA detection. This suggests that a well-calibrated $\lambda$ allows the model to incorporate meaningful structural cues without overwhelming the influence of individual node features. When $\lambda$ becomes excessively large, the F1 score exhibits a downward trend. This decline occurs because an overly strong emphasis on topological structure suppresses the contribution of node attribute features, leading to distorted representations. As a result, the model becomes less effective at distinguishing legitimate user queries from adversarial ones, ultimately impairing its detection capability. (3) The impact of $\lambda$ varies across datasets, suggesting dataset-dependent optimal values. While an optimal range of $\lambda \approx 0.5$ to $1$ is generally observed, the exact value that maximizes performance may differ based on dataset-specific properties such as graph sparsity, node connectivity, and query distribution patterns. This indicates that tuning $\lambda$ should be approached in a data-driven manner, potentially through cross-validation, to achieve the best trade-off between node attributes and topological information.

\section{Conclusion}
In this paper, we propose ATOM, a novel framework for detecting graph-based MEAs in GMLaaS environments under the transductive setting. To the best of our knowledge, we are the first to investigate the novel problem of detecting graph-based MEAs in GMLaaS. To address this problem, we design ATOM by focusing on real-time detecting and adaptive attacks. Specifically, we introduce sequential modeling and reinforcement learning to dynamically detect evolving attack patterns. We further conduct theoretical analysis for the query behavior and establish a theoretical foundation for our proposed framework. Extensive experiments on real-world datasets demonstrate ATOM’s superior performance over baselines in the real-time detecting scenario. Meanwhile, two future directions are worth further investigation. First, we focus on the simulated queries in this paper due to the limited availability of query datasets in GMLaaS. Thus, exploring large-scale industrial query datasets from real-world scenarios could provide a more accurate reflection of practical attack behaviors. Second, to better capture realistic user behavior, it is essential to investigate distributed adversaries who coordinate attacks across multiple accounts, which could provide valuable insights for enhancing detection mechanisms.


\bibliographystyle{ACM-Reference-Format}
\bibliography{main}

\appendix

\section{Proofs}
\begin{theorem}
  Consider a covering graph $\mathcal{D}$ in the graph $\mathcal{G}=(\mathcal{V},\mathcal{E})$, aiming to cover at least $\beta\in[0,1]$ percent nodes of $\mathcal{G}$, while minimizing $\sum_{u\in \mathcal{A}}w(u)$, where $w(u)$ is the weight of node $u$ and $\mathcal{A}$ represents the set of nodes not being covered, then the maximum covering percentage is given by \begin{align}\beta\le \min\{\frac{\abs{\mathcal{D}}-\frac{W}{\Bar{w_{\mathcal{A}}}}}{\frac{n}{\delta}-\frac{W}{\Bar{w_{\mathcal{A}}}}}, \frac{\abs{\mathcal{D}}\cdot \delta}{n}\}.\end{align} Here, $\abs{\mathcal{D}}$ represents the number of nodes in $\mathcal{D}$, $W$ is the total weight in $\mathcal{G}$, $\Bar{w_{\mathcal{A}}}$ represents the average weight in $\mathcal{A}$ and $\delta$ is the smallest degree for nodes in $\mathcal{D}$.
\end{theorem}

\begin{proof}
Since $\mathcal{D}$ at least covers $\beta$ of $\mathcal{G}$, we then get
\begin{equation}
     \abs{\mathcal{D}}\cdot\delta\ge \beta\cdot n.
\end{equation}
Also, 
\begin{equation}
    \abs{\mathcal{D}}\le \frac{\beta\cdot n}{\delta}+\frac{W_{\mathcal{A}}}{\Bar{w_{\mathcal{A}}}},
\end{equation}where $W_{\mathcal{A}}\le(1-\beta)\cdot n\cdot\Bar{w_{\mathcal{V}} }$ is defined as the total weight of $\mathcal{A}$, as $\Bar{w_{v}}$ representing the total average weight in $V$, $\Bar{w_{\mathcal{A}}}$ is the average weight in $A$. By (23) and (24), we directly get \begin{equation}
    \beta\le \frac{\abs{\mathcal{D}}\cdot \delta}{n}.
\end{equation}By solving (25) and $W_{\mathcal{A}}$ simultaneously, we obtain:\begin{align}
\abs{\mathcal{D}}-n\frac{\Bar{w_{\mathcal{V}}}}{\Bar{w_{\mathcal{A}}}}\le n\beta (\frac{1}{\delta}-\frac{\Bar{w_{\mathcal{V}}}}{\Bar{w_{\mathcal{A}}}})    
\end{align}
Then if $\frac{\Bar{w_{\mathcal{V}} }}{\Bar{w_{\mathcal{A}}}}> \frac{1}{\delta}$, \begin{equation}
    \beta\le \frac{\frac{\abs{\mathcal{D}}\cdot \delta}{n}-\delta\frac{\Bar{w_{\mathcal{V}} }}{\Bar{w_{\mathcal{A}}}}}{1-\delta\frac{\Bar{w_{\mathcal{V}} }}{\Bar{w_{\mathcal{A}}}}},
\end{equation}otherwise, i.e., $0<\frac{\Bar{w_{\mathcal{V}} }}{\Bar{w_{\mathcal{A}}}}< \frac{1}{\delta}$, we have \begin{equation}
    \beta\ge \frac{\frac{\abs{\mathcal{D}}\cdot \delta}{n}-\delta\frac{\Bar{w_{\mathcal{V}} }}{\Bar{w_{\mathcal{A}}}}}{1-\delta\frac{\Bar{w_{\mathcal{V}} }}{\Bar{w_{\mathcal{A}}}}}.
\end{equation}Thus, by (25), (27) and (28), \begin{align}
\beta\le \min\{\frac{\frac{\abs{\mathcal{D}}}{n}-\frac{\Bar{w_{\mathcal{V}} }}{\Bar{w_{\mathcal{A}}}}}{\frac{1}{\delta}-\frac{\Bar{w_{\mathcal{V}} }}{\Bar{w_{\mathcal{A}}}}}, \frac{\abs{\mathcal{D}}\cdot \delta}{n}\}.
\end{align}Introducing $\frac{\Bar{w_{\mathcal{V}} }}{\Bar{w_{\mathcal{A}}}}=\frac{\Bar{w_{\mathcal{V}} }n}{\Bar{w_{\mathcal{A}}}n}=\frac{W}{\Bar{w_{\mathcal{A}}}n}$, we finally finish the proof \begin{align}
\beta\le \min\{\frac{\abs{\mathcal{D}}-\frac{W}{\Bar{w_{\mathcal{A}}}}}{\frac{n}{\delta}-\frac{W}{\Bar{w_{\mathcal{A}}}}}, \frac{\abs{\mathcal{D}}\cdot \delta}{n}\}.
\end{align}
\end{proof}

\begin{proposition}
     Consider a changing graph $\mathcal{D}_{t-1}$ and $\mathcal{D}_{t}$ in the graph $\mathcal{G}=(\mathcal{V},\mathcal{E})$, where $\mathcal{D}_{t-1}\subset \mathcal{D}_{t}\subset \mathcal{G}$, achieving at least $\beta_{t-1}$ covering rate with the lowest degree $\delta_{t-1}$ and at most $\beta_{t}$ covering rate with the lowest degree $\delta_{t}$, respectively. Also, suppose that $A_{t-1}$ and $A_{t}$ represent the set of nodes that are not covered, with the corresponding weight $W_{{\mathcal{A}}_{t-1}}$, $W_{{\mathcal{A}}_{t}}$ and the average weight $\Bar{w_{A_{t}}}$, $\Bar{w_{A_{t-1}}}$. We then get \begin{align}
    \frac{\Delta W_{\mathcal{A}}}{\Delta\abs{\mathcal{D}}}\le\frac{(\beta_{t}-\beta_{t-1}) \cdot W}{\abs{\mathcal{D}_{t-1}}(1-\frac{\delta_{t-1}}{\delta_{t}})},
    \end{align}and $\delta_{t}>\delta_{t-1}$. Here, $\abs{\mathcal{D}_{t-1}},\abs{\mathcal{D}_{t}}$ represents the number of nodes in $\mathcal{D}_{t-1},\mathcal{D}_{t}$, $W$ is the total weight in $\mathcal{G}$.
\end{proposition}

\begin{proof}
Still, we directly have
\begin{align}
    \left\{\begin{array}{l} 
  \abs{\mathcal{D}_{t-1}}\cdot \delta_{t-1}\ge \beta_{t-1}\cdot n,  \\
  \abs{\mathcal{D}_{t}}\cdot \delta_{t}\le \beta_{t}\cdot n, 
\end{array}
    \right.
\end{align}and \begin{align}
    \left\{\begin{array}{l} 
  W_{{\mathcal{A}}_{t-1}}\le (1-\beta_{t-1})\cdot n\cdot \Bar{w_{\mathcal{V}} },  \\
  W_{{\mathcal{A}}_{t}}\ge (1-\beta_{t})\cdot n\cdot \Bar{w_{\mathcal{V}} }.
  \end{array}
    \right.
\end{align}By solving (33), we obtain:\begin{equation}
    0< W_{{\mathcal{A}}_{t-1}}-W_{{\mathcal{A}}_{t}}\le (\beta_{t}-\beta_{t-1})\cdot n\cdot \Bar{w_{\mathcal{V}} }.
\end{equation}

From (32) and assuming $\abs{\mathcal{D}_{t}}\cdot \delta_{t}\ge \abs{\mathcal{D}_{t-1}}\cdot \delta_{t-1}$, we have\begin{equation}
    \abs{\mathcal{D}_{t-1}}(1-\frac{\delta_{t-1}}{\delta_{t}})\le \abs{\mathcal{D}_{t}}-\abs{\mathcal{D}_{t-1}}\le n\cdot(\frac{\beta_{t}}{\delta_{t}}-\frac{\beta_{t-1}}{\delta_{t-1}})
\end{equation}By (34) and (35), we get\begin{equation}
    \frac{W_{{\mathcal{A}}_{t-1}}-W_{{\mathcal{A}}_{t}}}{\abs{\mathcal{D}_{t}}-\abs{\mathcal{D}_{t-1}}}\le \frac{(\beta_{t}-\beta_{t-1}) \cdot W}{\abs{\mathcal{D}_{t-1}}(1-\frac{\delta_{t-1}}{\delta_{t}})},
\end{equation}that is, \begin{equation}
    \frac{\Delta W_{\mathcal{A}}}{\Delta\abs{\mathcal{D}}}\le\frac{(\beta_{t}-\beta_{t-1}) \cdot W}{\abs{\mathcal{D}_{t-1}}(1-\frac{\delta_{t-1}}{\delta_{t}})},
\end{equation}which requires that\begin{align}
\delta_{t}>\delta_{t-1}.
\end{align}

\end{proof}

\begin{proposition}
     Consider changing graphs $\mathcal{D}_{t-1},\mathcal{D}_{t}$ and $\mathcal{D}_{t+1}$ in the graph $\mathcal{G}=(\mathcal{V},\mathcal{E})$, where $\mathcal{D}_{t-1}\subset \mathcal{D}_{t}\subset \mathcal{D}_{t+1}\subset \mathcal{G}$, achieving at least $\beta_{t-1}$ covering rate with the lowest degree $\delta_{t-1}$ and at most $\beta_{t}$ covering rate with the lowest degree $\delta_{t}$. Also, there is at least $\beta_{t}'$ covering rate and at most $\beta_{t+1}$ covering rate at time step $t$ and $t+1$. Suppose that $\mathcal{A}_{t-1},\mathcal{A}_{t}$ and $\mathcal{A}_{t+1}$ represent the set of nodes that are not covered, with the corresponding weight $W_{\mathcal{A}_{t-1}},W_{\mathcal{A}_{t}}$, $W_{\mathcal{A}_{t+1}}$ and the average weight $\Bar{w_{\mathcal{A}_{t-1}}},\Bar{w_{\mathcal{A}_{t}}}$, $\Bar{w_{\mathcal{A}_{t+1}}}$. We then get \begin{align}
    \frac{\Delta^2 W_{\mathcal{A}}}{\Delta^2\abs{\mathcal{D}}}\ge \left|\frac{W_{d}}{n}\frac{\Delta \delta_{t}}{\Delta \beta_{t}}-\frac{W\delta_{t+1}}{\abs{\mathcal{D}_{t}}}\frac{\Delta \beta_{t+1}}{\Delta \delta_{t+1}}\right|,
    \end{align}and $\delta_{t+1}>\delta_{t}>\delta_{t-1}$. Here, $\abs{\mathcal{D}_{t-1}},\abs{\mathcal{D}_{t}}$ and $\abs{\mathcal{D}_{t+1}}$ represent the number of nodes in $\mathcal{D}_{t-1},\mathcal{D}_{t}$ and $\mathcal{D}_{t+1}$, $W$ is the total weight in $\mathcal{G}$, assuming $W_{\mathcal{A}_{t-1}}-W_{\mathcal{A}_{t}}\ge W_{d}$.
\end{proposition}
\begin{proof}
    By the discussion from Proposition A.2, we have \begin{align}
    \frac{W_{{\mathcal{A}}_{t}}-W_{{\mathcal{A}}_{t+1}}}{\abs{\mathcal{D}_{t+1}}-\abs{\mathcal{D}_{t}}}\le \frac{(\beta_{t+1}-\beta_{t}') \cdot W}{\abs{\mathcal{D}_{t}}(1-\frac{\delta_{t}}{\delta_{t+1}})}.
    \end{align}Assuming that we have known $W_{{\mathcal{A}}_{t-1}}-W_{{\mathcal{A}}_{t}}\ge W_{d}$, then we get\begin{align}
    \frac{W_{{\mathcal{A}}_{t-1}}-W_{{\mathcal{A}}_{t}}}{\abs{\mathcal{D}_{t}}-\abs{\mathcal{D}_{t-1}}}\ge \frac{W_{d}}{n\cdot(\frac{\beta_{t}}{\delta_{t}}-\frac{\beta_{t-1}}{\delta_{t-1}})}.
    \end{align}The difference between two inequality is given by \begin{equation}
        \frac{\Delta^2 W_{\mathcal{A}}}{\Delta^2\abs{\mathcal{D}}}=\abs{\frac{W_{{\mathcal{A}}_{t-1}}-W_{{\mathcal{A}}_{t}}}{\abs{\mathcal{D}_{t}}-\abs{\mathcal{D}_{t-1}}}-\frac{W_{{\mathcal{A}}_{t}}-W_{{\mathcal{A}}_{t+1}}}{\abs{\mathcal{D}_{t+1}}-\abs{\mathcal{D}_{t}}}},
    \end{equation}which is restrict by \begin{align}
    \frac{\Delta^2 W_{\mathcal{A}}}{\Delta^2\abs{\mathcal{D}}}\ge \abs{\frac{W_{d}}{n\cdot(\frac{\beta_{t}}{\delta_{t}}-\frac{\beta_{t-1}}{\delta_{t-1}})}-\frac{(\beta_{t+1}-\beta_{t}') \cdot W}{\abs{\mathcal{D}_{t}}(1-\frac{\delta_{t}}{\delta_{t+1}})}}.
    \end{align}Observe that $\beta_{t}\frac{\delta_{t-1}}{\delta_{t}}+\beta_{t-1}\frac{\delta_{t}}{\delta_{t-1}}>2\beta_{t-1}$, then\begin{equation}
        \frac{\Delta^2 W_{\mathcal{A}}}{\Delta^2\abs{\mathcal{D}}}\ge \abs{\frac{W_{d}(\delta_{t}-\delta_{t-1})}{n(\beta_{t}-\beta_{t-1})}-\frac{W\delta_{t+1}(\beta_{t+1}-\beta_{t}')}{\abs{\mathcal{D}_{t}}(\delta_{t+1}-\delta_{t})}}.
    \end{equation}Introduce $\Delta \beta_{t}=\beta_{t}-\beta_{t-1}$ and $\Delta \delta_{t}=\delta_{t}-\delta_{t-1}$. We get \begin{equation}
        \frac{\Delta^2 W_{\mathcal{A}}}{\Delta^2\abs{\mathcal{D}}}\ge \abs{\frac{W_{d}\Delta \delta_{t}}{n\Delta \beta_{t}}-\frac{W\delta_{t+1}\Delta \beta_{t+1}}{\abs{\mathcal{D}_{t}}\Delta \delta_{t+1}}}.
    \end{equation}By Triangle Inequality, we finally have \begin{align}
    \frac{\Delta^2 W_{\mathcal{A}}}{\Delta^2\abs{\mathcal{D}}}\ge \left|\frac{W_{d}}{n}\frac{\Delta \delta_{t}}{\Delta \beta_{t}}-\frac{W\delta_{t+1}}{\abs{\mathcal{D}_{t}}}\frac{\Delta \beta_{t+1}}{\Delta \delta_{t+1}}\right|.
    \end{align}
\end{proof}
\begin{theorem}
    Consider changing graphs $\mathcal{D}_{t-1},\mathcal{D}_{t}$ and $\mathcal{D}_{t+1}$ with graph $\mathcal{G}=(\mathcal{V},\mathcal{E})$, where $\mathcal{D}_{t-1}\subset \mathcal{D}_{t}\subset \mathcal{D}_{t+1}\subset \mathcal{G}$, aiming at achieving at least $\beta_{t-1}$ covering rate with the lowest degree $\delta_{t-1}$ and at most $\beta_{t}$ covering rate with the lowest degree $\delta_{t}$. Also, there is at least $\beta_{t}'$ covering rate and at most $\beta_{t+1}$ covering rate at time step $t$ and $t+1$. Suppose that $\mathcal{A}_{t-1},\mathcal{A}_{t}$ and $\mathcal{A}_{t+1}$ represent the set of nodes that are not covered, with the corresponding weight $W_{\mathcal{A}_{t-1}},W_{\mathcal{A}_{t}}$, $W_{\mathcal{A}_{t+1}}$ and the average weight $\Bar{w_{\mathcal{A}_{t-1}}},\Bar{w_{\mathcal{A}_{t}}}$, $\Bar{w_{\mathcal{A}_{t+1}}}$. The second-order differences become essential , that is, $\frac{\Delta^2 W_{\mathcal{A}}}{\Delta^2\abs{\mathcal{D}}}\ge \frac{\Delta W_{\mathcal{A}}}{\Delta\abs{\mathcal{D}}}$ holds when \begin{align}
    W_{d}\ge \frac{n\delta_{t+1}\Delta \beta_{t}}{\abs{\mathcal{D}_{t-1}}\Delta \delta_{t}}(\frac{\Delta \beta_{t}}{\Delta \delta_{t}}+\frac{\Delta \beta_{t+1}}{\Delta \delta_{t+1}})W,
    \end{align}and $\delta_{t+1}>\delta_{t}>\delta_{t-1}$. Here, $\abs{\mathcal{D}_{t-1}},\abs{\mathcal{D}_{t}}$ and $\abs{\mathcal{D}_{t+1}}$ represent the number of nodes in $\mathcal{D}_{t-1},\mathcal{D}_{t}$ and $\mathcal{D}_{t+1}$, $W$ is the total weight in $\mathcal{G}$, assuming $W_{\mathcal{A}_{t-1}}-W_{\mathcal{A}_{t}}\ge W_{d}$.
\end{theorem}
\begin{proof}
    By proposition A.2 and A.3, we write out that \begin{equation}
        \frac{(\beta_{t}-\beta_{t-1}) \cdot W}{\abs{\mathcal{D}_{t-1}}(1-\frac{\delta_{t-1}}{\delta_{t}})}\le \left|\frac{W_{d}}{n}\frac{\Delta \delta_{t}}{\Delta \beta_{t}}-\frac{W\delta_{t+1}}{\abs{\mathcal{D}_{t}}}\frac{\Delta \beta_{t+1}}{\Delta \delta_{t+1}}\right|,
    \end{equation}which gives \begin{equation}
        W_{d}\ge \frac{n\Delta \beta_{t}W}{\Delta \delta_{t}}\left|\frac{\delta_{t}}{\abs{\mathcal{D}_{t-1}}}\frac{\Delta \beta_{t}}{\Delta \delta_{t}}+\frac{\delta_{t+1}}{\abs{\mathcal{D}_{t}}}\frac{\Delta \beta_{t+1}}{\Delta \delta_{t+1}}\right|.
    \end{equation}By Triangle Inequality, $W_{d}$ is required by \begin{equation}
        W_{d}\ge \frac{n\Delta \beta_{t}W}{\Delta \delta_{t}}(\frac{\delta_{t}}{\abs{\mathcal{D}_{t-1}}}\frac{\Delta \beta_{t}}{\Delta \delta_{t}}+\frac{\delta_{t+1}}{\abs{\mathcal{D}_{t}}}\frac{\Delta \beta_{t+1}}{\Delta \delta_{t+1}}).
    \end{equation}For simplicity, we can further get a tighter bound required by \begin{equation}
         W_{d}\ge \frac{n\delta_{t+1}\Delta \beta_{t}}{\abs{\mathcal{D}_{t-1}}\Delta \delta_{t}}(\frac{\Delta \beta_{t}}{\Delta \delta_{t}}+\frac{\Delta \beta_{t+1}}{\Delta \delta_{t+1}})W.
    \end{equation}
\end{proof}

\section{Reproducibility}
This section provides detailed descriptions of our datasets, experimental settings, and implementation details to ensure the reproducibility of our experiments. The full implementation, including code and configuration files, is available in our repository \href{https://github.com/LabRAI/ATOM}{https://github.com/LabRAI/ATOM}.

\subsection{Real-World Datasets.}
We conduct experiments using multiple widely adopted node classification datasets. The key statistics of these datasets are summarized in Table~\ref{tab:statistics}.

\begin{table}[t]
\small
\centering
\caption{Statistics of the adopted real-world graph datasets.}
\resizebox{0.49\textwidth}{!}{
\begin{tabular}{lcccc}
    \hline
                   & \textbf{\#Nodes}         & \textbf{\#Edges}       & \textbf{\#Attributes}       &\textbf{\#Classes}    \\
    \hline
    Cora                   &$2,708$          &$5,429$          &$1,433$          &$7$                   \\
    CiteSeer                 & $3,327$              & $4,723$              & $3,703$              & $6$                             \\
    PubMed           & $19,717$              & $88,648$              & $500$              & $3$                           \\
    Cornell       & $183$              & $298$              & $1703$              & $5$                              \\
     Wisconsin       & $251$              & $515$              & $1703$              & $5$                              \\
    \hline
\end{tabular}
}
\label{tab:statistics}
\end{table}

\subsection{Experimental Settings.}
For each real-world dataset in our experiment, we adopt Active-Learning-based MEAs to generate query sequences for the detection task. Here we set the hyperparameters in AL-based MEAs to be a wide list of values, where we present varying values of $\{1\%, 5\%, 10\%, 15\%\}$ percentage of the nodes in the graph as a prior knowledge and $\{35,70,105,140,200,300,400,500\}$ query budgets. We note that query budgets are always smaller than the nodes in the subgraph, and different sizes of the dataset will allow different numbers of query budgets. While the query sequence is generated, the fidelity of the extracted model is given to help label the sequence it is from. Generally, we label the sequence as an attacker if it corresponds to a fidelity larger than $0.65$ and a long query sequence, otherwise, we label it as a legitimate user if the query sequence is short or there is a fidelity smaller than $0.2$. We split all the query sequences from the same dataset with $70\%$ for training, $15\%$ for validating, and $15\%$ for testing. Only the sequence labels in the training set are visible for all models during training. For different datasets, the hyperparameters vary, but keep the same for the proposed model and its baselines. For all datasets in our experiment, we train a two-layer GCN by $200$ epochs as our GMLaaS system. And we adopt a learning rate of $0.01$ and a weight decay of $0.0005$ while training.

\subsection{Implementation of ATOM.}
ATOM is implemented based on Pytorch \cite{paszke2017automatic} with Adam optimizer \cite{kingma2014adam}. To ensure a fair and comprehensive evaluation, we conduct experiments using multiple random seeds and analyze model performance under different initialization conditions. Additionally, we perform an extensive hyperparameter search over ATOM’s parameter space, including the learning rate (lr), hidden state dimension, PPO clipping parameter, entropy coefficient, and lambda ($\lambda$). For consistency, the same number of hyperparameter searches is conducted for all baseline methods, and we report the best-performing configuration along with the standard deviation for each method. To accelerate training, we utilize four NVIDIA RTX 4090 GPUs for synchronous training, which significantly reduces the training time. However, it is important to note that different hardware configurations may lead to variations in reproducibility.

\subsection{Implementation of Baselines.}
\setlength{\parindent}{0pt}
\textbf{MLP. }We use a two-hidden-layer MLP for binary classification. Note that the MLP cannot process time series, we adopt the mean and the max features of the sequences instead.

\textbf{RNN-A, LSTM-A and Transformer-A. }We replace the sequential structure in ATOM with RNN, LSTM, and Transformer, named RNN-A, LSTM-A, and Transformer-A correspondingly, which thus allows us to classify the sequences. It should be mentioned that we adopt sine position coding and $4$ multi-head attention in the implementation of the Transformer-A.

\textbf{Crossformer, Autoformer, iTransformer, TimeNet, PatchTST and Informer. }We adopt a series of well-established baseline models in the field of time-series forecasting and classification. And we adopt their official open-source code for experiments\footnote{https://github.com/thuml/Time-Series-Library}.

\textbf{PRADA. }We adopt a statistical method as a baseline for comparison. We note that it is defined on static sequences, and we follow its official open-source code for experiments\footnote{https://github.com/SSGAalto/prada-protecting-against-dnn-model-stealing-attacks}.

\textbf{VarDetect. }We adopt an effective MEAs detection method based on Var as a baseline for comparison. We note that it is defined on static sequences and allows three types of outputs, specifically, VarDetect may output Alarm, Normal, and Uncertain. We follow its official open-source code for experiments\footnote{https://github.com/vardetect/vardetect}.

\subsection{Packages Required for Implementations.}
We perform the experiments mainly on a server with multiple Nvidia 4090 GPUs. We list the main packages with their versions in our repository.

\section{Supplementary Experiments}
\subsection{Evaluation of Detection Effectiveness}
In this subsection, we provide additional experimental results regarding the real-time detection effectiveness on different models and datasets. Specifically, we have shown the performance evolution over sequential query processing on Cora in Section~\ref{5.2}, and here we present more comprehensive results in Table~\ref{app4}, Table~\ref{app5}, and Table~\ref{app6}, with all other settings being consistent with the experiments presented in Section~\ref{5.2}. 

\begin{table}[h!]
\centering
\caption{Detection Performance with $25\%$ Query Sequences Across Different Datasets}
\label{app4}
\resizebox{\columnwidth}{!}{
\begin{tabular}{lccccc}
\hline
\textbf{Metrics} & \multicolumn{5}{c}{\textbf{F1 score}} \\
\cmidrule(lr){2-6} 
\textbf{Dataset} & \textbf{Wisconsin} & \textbf{Cornell} & \textbf{Cora} & \textbf{Citeseer} & \textbf{PubMed} \\
\hline
MLP & $16.67 \pm 8.17$ & $28.88 \pm 8.72$ & $13.73 \pm 8.69$ & $10.14 \pm 9.85$ & $12.44 \pm 8.41$ \\
RNN-A & $15.72 \pm 7.41$ & $18.17 \pm 7.74$ & $26.86 \pm 9.01$ & $29.54 \pm 8.78$ & $12.83 \pm 7.42$ \\
LSTM-A & $20.86 \pm 7.16$ & $23.64 \pm 7.69$ & $18.86 \pm 8.34$ & $21.26 \pm 9.11$ & $22.61 \pm 7.06$ \\
Transformer-A & $30.69 \pm 7.82$ & $22.34 \pm 8.63$ & $25.93 \pm 7.74$ & $23.39 \pm 9.14$ & $14.29 \pm 7.21$ \\
\hline
Crossformer & $13.72 \pm 6.11$ & $19.15 \pm 7.62$ & $20.37 \pm 7.32$ & $19.77 \pm 8.26$ & $17.64 \pm 9.34$ \\
Autoformer & $11.72 \pm 7.45$ & $12.34 \pm 6.78$ & $14.56 \pm 8.90$ & $16.78 \pm 9.01$ & $18.90 \pm 8.23$ \\
iTransformer & $13.45 \pm 8.79$ & $15.67 \pm 7.89$ & $17.89 \pm 7.34$ & $19.01 \pm 8.45$ & $21.23 \pm 9.56$ \\
TimesNet & $14.56 \pm 7.01$ & $16.78 \pm 8.90$ & $18.90 \pm 7.45$ & $21.23 \pm 9.56$ & $23.45 \pm 10.67$ \\
PatchTST & $15.67 \pm 6.87$ & $17.89 \pm 9.12$ & $20.12 \pm 6.23$ & $22.34 \pm 7.34$ & $24.56 \pm 9.45$ \\
Informer & $13.46 \pm 9.12$ & $15.58 \pm 8.65$ & $25.58 \pm 8.39$ & $20.55 \pm 6.52$ & $21.74 \pm 8.32$ \\
\hline
PRADA & $4.96 \pm 2.72$ & $6.25 \pm 2.94$ & $6.38 \pm 3.56$ & $5.52 \pm 2.23$ & $4.87 \pm 1.41$ \\
VarDetect & $17.23 \pm 10.31$ & $23.78 \pm 10.74$ & $25.91 \pm 7.90$ & $15.90 \pm 9.31$ & $29.65 \pm 10.34$ \\
\hline
ATOM & $\bm{34.91 \pm 6.42}$ & $\bm{28.12 \pm 6.83}$ & $\bm{27.14 \pm 7.79}$ & $\bm{34.59 \pm 6.41}$ & $\bm{25.47 \pm 3.49}$ \\
\hline

\hline
\end{tabular}
}
\end{table}

\begin{table}[h!]
\centering
\caption{Detection Performance with $50\%$ Query Sequences Across Different Datasets.}
\label{app5}
\resizebox{\columnwidth}{!}{
\begin{tabular}{lccccc}
\hline
\textbf{Metrics} & \multicolumn{5}{c}{\textbf{F1 score}} \\
\cmidrule(lr){2-6} 
\textbf{Dataset} & \textbf{Wisconsin} & \textbf{Cornell} & \textbf{Cora} & \textbf{Citeseer} & \textbf{PubMed} \\
\hline
MLP & $17.91 \pm 8.57$ & $34.44 \pm 9.36$ & $14.69 \pm 12.84$ & $23.73 \pm 10.08$ & $30.07 \pm 10.69$ \\
RNN-A & $29.71 \pm 7.72$ & $26.05 \pm 8.44$ & $33.74 \pm 18.16$ & $46.51 \pm 9.16$ & $29.89 \pm 9.64$ \\
LSTM-A & $34.42 \pm 8.26$ & $34.67 \pm 8.11$ & $34.16 \pm 19.43$ & $31.42 \pm 8.85$ & $31.42 \pm 9.12$ \\
Transformer-A & $54.58 \pm 6.21$ & $7.55 \pm 13.23$ & $36.07 \pm 10.14$ & $36.74 \pm 9.29$ & $32.26 \pm 8.21$ \\
\hline
Crossformer & $37.89 \pm 8.95$ & $35.77 \pm 8.76$ & $35.75 \pm 8.85$ & $34.61 \pm 7.22$ & $36.51 \pm 8.38$ \\
Autoformer & $34.44 \pm 9.82$ & $33.73 \pm 8.73$ & $23.71 \pm 8.55$ & $30.06 \pm 7.24$ & $29.72 \pm 8.74$ \\
iTransformer & $35.55 \pm 7.90$ & $38.25 \pm 8.41$ & $37.00 \pm 9.00$ & $36.15 \pm 7.12$ & $39.75 \pm 6.50$ \\
TimesNet & $32.10 \pm 8.05$ & $31.85 \pm 7.25$ & $31.78 \pm 9.30$ & $31.52 \pm 6.50$ & $34.35 \pm 8.75$ \\
PatchTST & $37.10 \pm8.10$ & $35.50 \pm 7.50$ & $36.00 \pm 8.70$ & $34.85 \pm 8.85$ & $35.40 \pm 6.42$ \\
Informer & $37.91 \pm 8.98$ & $35.78 \pm 8.78$ & $35.76 \pm 7.86$ & $34.62 \pm 6.23$ & $46.52 \pm 7.39$ \\
\hline
PRADA & $9.17 \pm 3.03$ & $7.45 \pm 3.52$ & $8.20 \pm 4.77$ & $8.08 \pm 4.01$ & $7.95 \pm 4.26$ \\
VarDetect & $39.12 \pm 11.44$ & $33.34 \pm 7.07$ & $30.00 \pm 11.76$ & $36.79 \pm 7.92$ & $49.19 \pm 11.91$ \\
\hline
ATOM & $\bm{49.77 \pm 6.65}$ & $\bm{49.18 \pm 5.39}$ & $\bm{45.31 \pm 6.94}$ & $\bm{40.25 \pm 4.48}$ & $\bm{35.49 \pm 7.87}$ \\
\hline
\end{tabular}
}
\end{table}

\begin{table}[h!]
\centering
\caption{Detection Performance with $75\%$ Query Sequences Across Different Datasets.}
\label{app6}
\resizebox{\columnwidth}{!}{
\begin{tabular}{lccccc}
\hline
\textbf{Metrics} & \multicolumn{5}{c}{\textbf{F1 score}} \\
\cmidrule(lr){2-6} 
\textbf{Dataset} & \textbf{Wisconsin} & \textbf{Cornell} & \textbf{Cora} & \textbf{Citeseer} & \textbf{PubMed} \\
\hline
MLP & $20.61 \pm 9.74$ & $45.37 \pm 9.44$ & $18.03 \pm 10.13$ & $33.26 \pm 13.41$ & $43.82 \pm 9.87$ \\
RNN-A & $56.30 \pm 9.21$ & $43.82 \pm 8.87$ & $ 45.40 \pm 8.18$ & $50.01 \pm 8.21$ & $43.50 \pm 8.54$ \\
LSTM-A & $44.78 \pm 9.84$ & $37.75 \pm 8.54$ & $48.15 \pm 9.06$ & $47.62 \pm 10.65$ & $42.55 \pm 8.23$ \\
Transformer-A & $48.51 \pm 8.86$ & $49.76 \pm 7.80$ & $47.84 \pm 7.78$ & $41.14 \pm 8.46$ & $59.62 \pm 7.97$ \\
\hline
Crossformer & $47.75\pm7.85$ & $55.21\pm8.98$ & $52.69\pm9.53$ & $46.43\pm9.34$ & $41.76\pm10.47$ \\
Autoformer & $42.18\pm7.25$ & $63.47\pm8.54$ & $50.12\pm8.01$ & $45.98\pm9.76$ & $61.42\pm8.89$ \\
iTransformer & $41.52\pm7.34$ & $46.89\pm9.15$ & $54.37\pm7.78$ & $47.12\pm8.98$ & $51.23\pm7.56$ \\
TimesNet & $41.96\pm8.12$ & $72.43\pm7.21$ & $55.87\pm8.87$ & $45.74\pm8.15$ & $57.18\pm8.63$ \\
PatchTST & $45.34\pm8.56$ & $58.12\pm8.45$ & $61.23\pm7.02$ & $48.34\pm7.34$ & $59.12\pm8.01$ \\
Informer & $44.72\pm7.98$ & $67.77\pm9.21$ & $65.71\pm8.12$ & $44.07\pm9.71$ & $59.76\pm7.19$ \\
\hline
PRADA & $13.45\pm1.56$ & $10.04\pm0.87$ & $10.15\pm1.24$ & $10.16\pm2.56$ & $14.56\pm1.12$ \\
VarDetect & $49.92\pm9.13$ & $48.37\pm9.72$ & $53.13\pm7.07$ & $47.17\pm7.43$ & $50.79\pm6.94$ \\
\hline
ATOM & $\bm{74.13\pm4.46}$ & $\bm{69.32\pm5.39}$ & $\bm{77.52\pm4.10}$ & $\bm{60.48\pm5.65}$ & $\bm{71.91\pm4.81}$ \\
\hline
\end{tabular}
}
\end{table}

\subsection{Evaluation of Ablation Study}
In this subsection, we provide additional experimental results for the ablation study. Specifically, we have presented F1 scores from the ablated model on Cora, PubMed, and CiteSeer in Table~\ref{table_ablation}. Here we show the other two datasets named Wisconsin and Cornell in Table~\ref{tab:ablation} to show the generalization of our experiments.
\begin{table}[t]
\small
\centering
\caption{F1 scores from the ablated model. The best results are highlighted in bold.}
\resizebox{0.49\textwidth}{!}{
\begin{tabular}{lccccc}
    \hline
    \textbf{Model}               & \textbf{Cornell}         & \textbf{Cora}       & \textbf{PubMed}       &\textbf{Wisconsin}  &\textbf{Citeseer}   \\
    \hline
    ATOM                   &\bm{$89.66$}          &\bm{$86.88$}          &\bm{$83.24$}          &\bm{$81.48$}          &\bm{$78.89$}         \\
    Standard GRU                 & $85.19$              & $80.64$              & $75.47$              & $73.68$              & $71.97$                 \\
    Simple Embeddings            & $61.90$              & $67.92$              & $54.83$              & $73.41$              & $60.87$              \\
    No Mapping Matrix       & $61.31$              & $81.54$              & $79.94$              & $75.93$              & $74.71$                 \\
    \hline
\end{tabular}
}
\label{tab:ablation}
\end{table}

\section{Probabilistic Interpretation of ATOM}
In this section, we provide a probabilistic interpretation of our model. In particular, we consider the query sequences generated by users and show how they can be organized and analyzed as a stochastic process.

\begin{definition}[Query Lists]
    For each user $u_{i}$, suppose the observed query sequence is $\mathcal{Q}_{i}=\{q_{i,1},q_{i,2},\cdots,q_{i,T_i}\}$. From each sequence $\mathcal{Q}_{i}$, we construct a corresponding query list $l_i$ by sequentially connecting consecutive queries. That is, for each $p\in\{2,3,\cdots,T_i\}$ we connect $q_{i,p-1}$ to $q_{i,p}$ with an edge whose weight $w_{i,p-1}$ is defined as the length of the shortest path from $q_{i,p-1}$ to $q_{i,p}$ on the graph $\mathcal{G}$. We denote the number of nodes in $l_i$ by its length $\abs{l_i}$.
\end{definition}

\begin{proposition}
    Consider the graph $\mathcal{G}=(\mathcal{V},\mathcal{E})$ used for training in GMLaaS under the transductive setting. Assume that, due to the diversity of prior knowledge among normal users, every node in $\mathcal{G}$ is eventually visited by some user. Then there exists a collection of query lists generated by normal users whose union covers $\mathcal{V}$ and which are pairwise disjoint (i.e., no two lists share any node). In fact, if we denote by $l_{\text{min}}$ a query list having the minimum number of nodes, then by the pigeonhole principle the maximum number of pairwise disjoint query lists is bounded by $J=\lceil \frac{\abs{\mathcal{G}}}{l_{\text{min}}} \rceil$.
\end{proposition}
\begin{proof}
Since all nodes in $\mathcal{G}$ are visited by some normal user, we can extract query lists so that every node appears in at least one list. Choosing one list $l_{\text{min}}$ that is shortest (i.e., has the minimum number of nodes), note that any collection of pairwise disjoint query lists must assign at least $\abs{l_{\text{min}}}$ distinct nodes to each list. Hence, by the pigeonhole principle the number of such disjoint lists is at most $\lceil \frac{\abs{\mathcal{G}}}{l_{\text{min}}} \rceil$.
\end{proof}

Let us now denote this upper bound by $J=\lceil \frac{\abs{\mathcal{G}}}{l_{\text{min}}} \rceil$. We select a collection of$J$ query lists $\{l_i\}_{i=1}^J$ from the normal users. To facilitate further analysis, we pad each query list so that every list has the same length. Specifically, let $k=\max\{\abs{l_{1},\abs{l_2}\cdots,\abs{l_{J}}}\}$. For any query list $l_{i}=\{q_{i,1}, q_{i,2}, \cdots, q_{i,T_{i}}\}$ with $\abs{l_{i}}<k$, we extend it by replicating its last query $q_{i,T_{i}}$ for positions $T_{i+1},T_{i+2},\cdots,k$ and assign an edge weight of $0$ to each newly introduced edge. This padding ensures that each list $l_i$ is represented as a sequence of exactly $k$ queries. Observe that for the $(J+1)$th query list, every query it contains already appears in one of the first $J$ lists. Thus, we can regard the generation of query sequences as a stochastic process over the collection $\{l_i\}_{i=1}^J$. We now introduce several definitions that formalize this process.

\begin{definition}[List Distance]
    For any two padded query lists \begin{align}
    l_{i}=\{q_{i,1},q_{i,2},\cdots,q_{i,k}\}\quad l_{j}=\{q_{j,1},q_{j,2}\cdots,q_{j,k}\},
    \end{align}the distance between $l_i$ and $l_j$ is defined as\begin{align}
    d(i,j)=\sum_{s=1}^k|q_{i,s}\rightarrow q_{j,s}|,
    \end{align}where $|q_{i,s}\rightarrow q_{j,s}|$ denotes the length of the shortest path on $\mathcal{G}$ between the $s$th query of $l_{i}$ and the $s$th query of $l_j$.
\end{definition}

\begin{definition}[List Transition Probability]
    Given the distance $d(i,j)$ and a sensitivity parameter $\lambda_{s}>0$, the probability of transitioning from query list $l_i$ to query list $l_j$ is defined according to the Boltzmann distribution as \begin{align}
    p_{ij}^{(\text{state})}=\frac{e^{-\lambda_sd(i,j)}}{\sum_{r=1}^Je^{-\lambda_sd(i,r)}}.
    \end{align}
\end{definition}

\begin{definition}[Query Distance]
    Within a given query list $l_i$ (with associated edge weights $\{w_{i,1},w_{i,2},\cdots,w_{i,k-1}\}$), the distance between queries at positions $s$ and $q$ is defined as \begin{align}
    d_{n}(s,q)=\sum_{r=\min\{s,q\}}^{\max\{s,q\}-1}w_{i,r},
    \end{align}with the convention that $d_n(s,s)=0$.
\end{definition}

\begin{definition}[Query Transition Probability]
    Let $\lambda_n$ be a local sensitivity parameter. Then, for a given query list $l_i$, the probability of transitioning from the query at position $s$ to the query at position $q$ is given by \begin{align}
    p_{sq}^{\text{(query)}}=\frac{e^{-\lambda_nd_n(s,q)}}{\sum_{t=1}^ke^{-\lambda_nd(s,t)}}.
    \end{align}
\end{definition}

Before proceeding, we relabel the query lists as follows. Suppose that the first query from the $(J+1)$th list is observed in one of the initial $J$ lists; then we designate that list as $l_1$. The remaining lists are then relabeled as $l_2,l_3,\cdots,l_J$ in order according to their proximity (as measured by $d(i,1)$) to $l_1$.

Next, we define a composite state as an ordered pair $(i,q)$, where $i\in\{1,2,\cdots,J\}$ indicates the query list, and $q\in\{1,2,\cdots,k\}$ indicates the position within that list. We assume that a new query behavior always starts from a fixed initial composite state: \begin{align}
\pi^{(0)}(i,q)=\delta_{i1}\delta_{q1},
\end{align}where $\delta$ is the Kronecker delta.

Then, we define the one-step transition probability from a composite state $(i,q)$ to another composite state $(j,s)$ as the product of the list-level and query-level transition probabilities:\begin{align}
P_{(i,q)\rightarrow(j,s)}=p_{ij}^{(\text{state})}\cdot p_{sq}^{\text{(query)}}.
\end{align}We note that under this formulation the probability of reaching any given composite state after a sequence of transitions reflects the likelihood that the observed query behavior is generated by a normal user. In particular, by assigning higher transition probabilities to paths corresponding to smaller distances, the model implicitly favors query sequences that are more "normal."

\begin{corollary}
    With the initial composite state fixed as $(i_0,q_0)=(1,1)$, consider the query behavior as a stochastic process. Then the probability of reaching a composite state $(i_K,q_K)$ after $K$ transitions is given by\begin{align}
    \pi^{(K)}(i_K,q_K)=\sum_{(i_1,q_1),\cdots,(i_{K-1},q_{K-1})}\prod_{n=0}^{K-1}P_{(i_n,q_n)\rightarrow(i_{n+1},q_{n+1})},
    \end{align}where the sum is taken over all possible sequences of intermediate composite states.
\end{corollary}

By combining the list-level transition process with the local (within-list) query transition process, our composite model assigns a well-defined probability to the event that a query sequence (starting from a fixed initial query, e.g., the first query of $l_1$) evolves through a series of transitions to reach a specified composite state $(i,q)$. This probabilistic framework not only captures the behavior of normal users but also underpins the ATOM mechanism, thereby providing strong interpretability to our attack detection strategy.
\end{document}